\documentclass[twoside,11pt]{article}

\usepackage{blindtext}

\usepackage[english]{babel}

\usepackage[top=1in,bottom=1in,left=1in,right=1in]{geometry}
\usepackage{amsmath,amssymb}
\usepackage{mathtools}
\usepackage{dsfont}
\usepackage[numbers]{natbib}
\usepackage{amsthm}

\usepackage{graphicx}
\usepackage{hyperref}
\usepackage{cleveref}
\usepackage{comment}
\usepackage[dvipsnames]{xcolor}

\usepackage{subcaption}

\newtheorem{theorem}{Theorem}
\newtheorem{lemma}{Lemma}
\newtheorem{corollary}{Corollary}
\newtheorem{definition}{Definition}
\newtheorem{proposition}{Proposition}
\newtheorem{assumption}{Assumption}
\newtheorem{remark}{Remark}
\newtheorem{example}{Example}

\hypersetup{citecolor=purple, colorlinks=true, linkcolor=blue}

\usepackage{graphicx}

\newcommand{\bc}{\mathbf{c}}
\newcommand{\C}{\mathbf{C}}
\newcommand{\x}{\mathbf{x}}

\newcommand{\mup}{\boldsymbol{\mu_+}}
\newcommand{\mum}{\boldsymbol{\mu_-}}
\newcommand{\f}{\boldsymbol{f}}
\newcommand{\mupm}{\boldsymbol{\mu_\pm}}
\newcommand{\fp}{\boldsymbol{f_+}}
\newcommand{\fm}{\boldsymbol{f_-}}
\newcommand{\fpm}{\boldsymbol{f_\pm}}

\newcommand{\p}[1]{\left(#1 \right)}

\newcommand{\defeq}{\vcentcolon=}
\newcommand{\eqdef}{=\vcentcolon}

\newcommand{\ones}{\mathbf{1}}
\newcommand{\dd}{\mathrm{d}}
\newcommand{\OT}{\mathrm{OT}}

\newcommand{\supp}{\mathrm{supp}}
\newcommand{\id}{\mathrm{id}}
\newcommand{\dom}{\mathrm{dom}}
\newcommand{\osc}{\mathrm{osc}}

\DeclareMathOperator*{\argmin}{arg\,min}

\renewcommand{\phi}{\varphi}
\newcommand{\eps}{\varepsilon}

\newcommand{\E}{\mathbb{E}}

\newcommand{\N}{\mathbb{N}}

\renewcommand{\P}{\mathbb{P}}
\newcommand{\Q}{\mathbb{Q}}
\newcommand{\R}{\mathbb{R}}

\newcommand{\cF}{\mathcal{F}}
\newcommand{\cG}{\mathcal{G}}

\newcommand{\cP}{\mathcal{P}}
\newcommand{\cQ}{\mathcal{Q}}
\newcommand{\cR}{\mathcal{R}}
\newcommand{\cS}{\mathcal{S}}

\newcommand{\cW}{\mathcal{W}}
\newcommand{\cX}{\mathcal{X}}
\newcommand{\cY}{\mathcal{Y}}
\newcommand{\cZ}{\mathcal{Z}}

\newcommand{\bbE}{\mathbb{E}}

\newcommand{\bbP}{\mathbb{P}}

\newcommand{\bbR}{\mathbb{R}}

\newcommand{\bx}{\boldsymbol{x}}

\newcommand{\class}[1]{\mathcal{#1}}

\newcommand{\fbayes}{f^{\text{Bayes}}}
\newcommand{\gbayes}{g^{\text{Bayes}}}

\usepackage{authblk}

\providecommand{\keywords}[1]
{
  \small	
  \textbf{\textit{Keywords---}} #1
}

\begin{document}

\title{Demographic parity in regression and
classification within the unawareness framework}

\author[1]{Vincent Divol \thanks{vincent.divol@ensae.fr}}
\author[1]{Solenne Gaucher \thanks{solenne.gaucher@ensae.fr}}
\affil[1]{CREST, ENSAE, IP Paris}
\maketitle

\begin{abstract}
This paper explores the theoretical foundations of fair regression under the constraint of demographic parity within the unawareness framework, where disparate treatment is prohibited, extending existing results where such treatment is permitted. Specifically, we aim to characterize the optimal fair regression function when minimizing the quadratic loss. Our results reveal that this function is given by the solution to a barycenter problem with optimal transport costs. Additionally, we study the connection between optimal fair cost-sensitive classification, and optimal fair regression. We demonstrate that nestedness of the decision sets of the classifiers is both necessary and sufficient to establish a form of equivalence between classification and regression. Under this nestedness assumption, the optimal classifiers can be derived by applying thresholds to the optimal fair regression function; conversely, the optimal fair regression function is characterized by the family of cost-sensitive classifiers.
\end{abstract}


\section{Introduction}
\subsection{Motivation}

Recent breakthroughs in artificial intelligence have led to the widespread adoption of machine learning algorithms, exerting an increasingly influential and insidious impact on our lives. Essentially, these algorithms learn to detect and reproduce patterns using massive datasets. It is now widely recognized that these predictions carry the risk of perpetuating, or even exacerbating, the social discriminations and biases often present in these datasets \citep{Propublica, barocas-hardt-narayanan}. Algorithmic fairness seeks to measure and mitigate the unfair impact of algorithms; we refer the reader to the reviews by \cite{barocas-hardt-narayanan, delbarrio2020review, Oneto_2020} for an introduction.

Different approaches have been developed to mitigate algorithmic unfairness. One approach focuses on \textit{individual fairness}, ensuring that similar individuals are treated similarly, regardless of potentially discriminatory factors. Another approach targets \textit{group fairness}, aiming to prevent algorithmic  predictions
from discriminating against groups of individuals.
\textit{Statistical fairness} falls under the latter approach and relies on the formalism of supervised learning to impose fairness criteria while minimizing a risk measure. In this work, we study risk minimization under the demographic parity criterion, which requires that predictions be statistically independent of sensitive attributes. Although this criterion, introduced by \cite{Calders09, pmlr-v97-agarwal19d}, has some known limitations \citep{Hardt16, JMLR:v20:18-262}, it finds application in a wide range of scenarios \citep{Makhlouf23, JMLR:v25:23-0322}. Its simplicity arguably makes it the most extensively studied criterion.

%

The statistical fairness literature can be broadly divided into two currents, depending on whether the direct use of the protected attribute in predictions is permitted or not. A first line of works, studying the \textit{awareness framework}, considers regression functions that make explicit use of discriminating attributes, thus treating individuals differently based on discriminating factors. For this reason, this approach is also often referred to as \textit{disparate treatment}. In this work, we adopt the \textit{unawareness framework}, in which disparate treatment is prohibited and the regression function cannot directly use the sensitive attribute. Empirical evidence from simulations \citep{liptonDisparity} indicates that within the unawareness framework, predictions often result in suboptimal trade-offs between fairness and accuracy and may induce within-group discrimination. Moreover, the authors conjecture that while the unawareness framework aims to prevent discrimination based on sensitive attributes, predictions in this setting implicitly rely on estimates of these attributes—a phenomenon later proven in \cite{pmlr-v206-gaucher23a} for classification problems. Nevertheless, this framework remains crucial in practice, as the direct use of sensitive attributes may be legally prohibited or simply unavailable at prediction time.

In this paper, we investigate the problem of fair regression under demographic parity constraints within the unawareness framework. A key difficulty in overcoming algorithmic unfairness is the limited understanding of how fair algorithms make predictions. Therefore, we focus on providing a simple mathematical characterization of the optimal regression function in the presence of fairness constraints. 

\subsection{Problem statement}  
Let $(X,S,Y)$ be a tuple in $\cX \times \cS \times \R$ with distribution $\P$, where $X$ corresponds to a non-sensitive feature in a feature space $\cX$, $S$ is a sensitive attribute in a finite set $\cS$, and $Y$ is a response variable that we want to predict, which has a finite second moment. To illustrate this problem with an example, assume, as in \cite{Chzhen2020AMF}, that $X$ represents a candidate's skill, $S$ is an attribute indicating groups of populations, and $Y$ is the current market salary of the candidate. Due to historical biases, the distribution of the salary may be unbalanced between the groups. Our aim is to make predictions that are fair, and as close as possible to the current market value $Y$. In the unawareness framework, we cannot make explicit use of the sensitive attribute to make our predictions. Therefore, we consider regression functions of the form $f :  \cX \rightarrow \bbR$ in the set of score functions $\cF$. We want to ensure that our regression function satisfies the following demographic parity criterion.
\begin{definition}[Demographic parity]
The function $f :  \cX \rightarrow \bbR$ verifies the Demographic Parity criterion if 
\begin{align*}
    f(X) \perp S.
\end{align*}
\end{definition}
In essence, the demographic parity criterion requires that the distribution of predictions (in our example, the salary) be identical across all groups. We assess the quality of a regression function $f$ through its quadratic risk 
\begin{align*}
    \cR_{sq}(f) = \mathbb{E}\left[\left(Y - f(X)\right)^2\right].
\end{align*}
\begin{definition}[Fair regression]\label{def:fair_regression}
An optimal fair regression function $f^*$ satisfies
\begin{align}\label{eq:unaware_regression}
    f^* \in \argmin_{f \in \cF}\left \{\cR_{sq}(f)\ : \ f(X) \perp S \right\},
\end{align}
where $\cF$ is the set of regression functions from $\cX$ to $\R$. 
\end{definition}
Classical results show that when no fairness constraints are imposed, the Bayes regression function $\fbayes$ minimizing the squared risk $\cR_{sq}$ is a.s. equal to the conditional expectation $\eta$, where 
\begin{align*}
   \eta(x) = \mathbb{E}\left[Y \vert X = x\right].
\end{align*}
In this paper, we also investigate the relationship between classification and regression problem. When $Y \in \{0,1\}$ a.s., the quality of a classification function $g : \cX \rightarrow \{0,1\}$ can be assessed through its expected weighted $0-1$ loss $\cR_{y}(g)$, where for $y \in [0,1]$, $\cR_{y}(g)$ is defined as
\begin{align*}
    \cR_{y}(g) = y\cdot \mathbb{P}\left[Y = 0,  g(X) = 1\right] + (1-y)\cdot \mathbb{P}\left[Y = 1,  g(X) = 0\right].
\end{align*}
For the choice $y = 1/2$, minimizing this risk measure corresponds to maximizing the classical accuracy measure. 
\begin{definition}[Fair classification]\label{def:fair_classif}
For a given value $y\in [0,1]$, an \textit{optimal fair classification function} $g^*_y$ verifies
\begin{align}\label{eq:unaware_classif}
    g^*_y \in \argmin_{g \in \cG}\left \{\cR_{y}(g)\ : \ g(X) \perp S \right\},
\end{align}
where $\cG$ is the set of classification
 functions from $\cX$ to $\{0,1\}$. 
\end{definition}
Let us again illustrate this problem with an example from recruitment. Assume that $X$ represents a candidate's skill, $S$ is an attribute indicating different population groups, and $Y$ denotes whether a human recruiter would consider the candidate qualified for a given position. Due to historical biases, the distribution of the binary response $Y$ may be unbalanced across the groups. We aim to make a prediction, or equivalently take the decision to accept or reject a candidate. Our goal is to make predictions for the value of $Y$, or equivalently, to decide whether to accept or reject a candidate, in a way that is both accurate and fair. Specifically, under demographic parity, we aim to ensure that the probability of acceptance is the same across all groups.

Classical results show that when no fairness constraints are imposed, the classifier $\gbayes_y(x) = \mathds{1}\left\{\fbayes(x)  \geq y \right\}$ is a Bayes classifier that minimizes $\cR_{y}(g)$. This relationship is at the heart of the design and study of plug-in classifiers \citep{Yang1999,Massart_2006,Tsybakov2007,Biau2008}. Interestingly, it was recently shown that a similar relationship holds under demographic parity constraints in the awareness framework \cite{pmlr-v206-gaucher23a}. Extending this result to the unawareness framework has remained an open problem, which we address in this paper.


\paragraph{Notation} We first set some notation.  Recall that we are given a tuple $(X,S,Y)$ in $\cX\times \cS\times \R$ with distribution $\P$, where $\cX$ is any measurable space (the space of features) and $\cS$ is a finite set (the set of labels). For $s\in \cS$, we denote by $p_s$ the probability $\P(S=s)$ and by $\mu_s$ the conditional law of $X|S=s$. We let $\mu =  \sum_{s\in \cS} p_s\mu_s$ be the marginal distribution of $X$. We let $\cP(\cX)$ be the set of probability measures on the measurable space $\cX$. Moreover, we let $L^1(\nu)$  be the space of functions integrable with respect to the probability measure $\nu$. Finally, $\mathring C$ denotes the interior of the set $C$.

\subsection{Related work}
\paragraph{Fair classification} Research on optimal prediction under demographic parity constraints has primarily focused on classification, where the goal is to predict a binary response in $\{0,1\}$, as this problem is intrinsically linked to the issue of fair candidate selection, central in algorithmic fairness. This problem is well understood in the awareness setting from an algorithmic point of view \citep{Feldman15,pmlr-v81-menon18a,NEURIPS2020_29c0605a, pmlr-v161-schreuder21a,NEURIPS2022_b1d9c7e7,JMLR:v25:23-0322}. On the theoretical side, \cite{pmlr-v206-gaucher23a}  recently proved  that the optimal classifier for the risk $\cR_{y}$ can be obtained as the indicator that the optimal fair prediction function for the squared loss  $f^*$ is above the threshold $y$, a result that was later extended in \cite{pmlr-v202-xian23b} to multi-class classification. 

Less is known about fair classification in the unawareness framework. On the algorithmic side, several works have proposed various relaxations of the demographic parity constraint, leading to tractable algorithms for computing  classifiers \citep{NIPS2016_dc4c44f6, JMLR:v20:18-262,Oneto2020}. On the theoretical side, \cite{liptonDisparity} provided empirical evidence suggesting that fair classifiers may base their decisions on non-relevant features correlated with the sensitive attribute, potentially disrupting within-group ordering. This hypothesis was further confirmed by \cite{pmlr-v206-gaucher23a}, who characterized the optimal fair classifier in the unawareness framework. They showed that it is given by the indicator that the conditional expectation $\eta(X)$ is above a threshold, which depends on the probabilities that the individual described by $X$ belongs to the different groups. Notably, the question of whether this classifier can be obtained by thresholding the optimal fair prediction function for the squared loss remains an open problem.

\paragraph{Fair regression} In the awareness framework, fair regression is well understood  from both the algorithmic and theoretical points of view \citep{Chzhen2020AMF,gouic2020projection,Chzhen2020a}. On the theoretical front, it has been shown that the problem of fair regression under demographic parity can be rephrased as the problem of finding the weighted barycenter of the distributions of $\eta(X,S)=\E[Y|X,S]$ across different groups, with costs given by optimal transport problems.
\begin{theorem}[\cite{Chzhen2020AMF,gouic2020projection}]\label{thm:awareness}
    Assume that for all $s \in \cS$, the distribution $\nu_{s}$ of $\eta(X,S)$ for $S=s$ has no atoms, and let $p_s = \mathbb{P}(S = s)$. Then,
    \begin{align*}
        \underset{\text{f is fair}}{\min} \cR_{sq}(f) = \min_{\nu\in \cP(\R)}\sum_{s\in \cS}p_s\cW_2^2(\nu_s, \nu)
    \end{align*}
    where $\cW_2^2(\nu_s, \nu)$ is the squared Wasserstein distance between $\nu_s$ and $\nu$. Moreover, if $f^*$ and $\nu$ solve the left-hand side and the right-hand side problems respectively, then $\nu$ is equal to the distribution of $f^*(X,S)$, and 
    \begin{align*}
        f^*(x,s) = \left(\sum_{s' \in \cS} p_{s'}\cQ_{s'}\right)\circ F_s(\eta(x,s)),
    \end{align*}
    where $\cQ_s$ and $F_s$ are respectively the  quantile function and the c.d.f. of $\nu_s$.
\end{theorem}
This result relates the problem of fair regression in the awareness framework to  a more general optimal transport problem. Interestingly, this problem has an explicit solution, given by the quantiles and c.d.fs of the conditional expectation $\eta(X,S)$ across the different groups. This explicit formulation yields, as an immediate consequence, that the optimal fair regression function preserves order, a property introduced in \cite{Chzhen2020a, Chzhen2020AMF} within the awareness framework. Recall that the Bayes regression function in the awareness framework is $\eta$. A prediction function $f$ is said to \textit{preserve order} if for any two candidates $(x,x') \in \cX^2$ in the same group $s \in \cS$, $\eta(x,s) \leq \eta(x',s)$ implies $f(x, s) \leq f(x',s)$. Thus, this property implies that the fairness correction does not alter the ordering of the predictions within a group.

In contrast, the problem of fair regression within the unawareness framework has been seldom studied, particularly from a theoretical perspective. One reason for this is that the demographic parity constraint is more challenging to implement without disparate treatment. While algorithms complying with these constraints have been proposed by \cite{Chzhen2020AnEO} and \cite{zhou2023groupblind},  the authors do not claim that the estimators obtained are optimal in terms of risk. \cite{pmlr-v97-agarwal19d} propose an algorithm based on a discretization of the problem, followed by a reduction to cost-sensitive classification. However, their algorithm requires calling an oracle cost-sensitive classifier, which may not be available in practice. Additionally, their results are limited to a class of regression functions with bounded Rademacher complexity.

\subsection{Outline and contribution}

In this paper, we focus on the theoretical aspects of the problem of fair regression in the unawareness framework, specifically on characterizing and studying the optimal regression function. We extend  results presented earlier in the awareness framework to this setting, albeit under the assumption that the sensitive attribute takes only two values; henceforth, we assume that $\cS = \{1,2\}$. Although restrictive, this assumption is not uncommon in the literature \citep{liptonDisparity} and covers the important case where one of the two groups includes protected individuals. Our results shed light on important phenomena, and we leave the extension to scenarios with more than two groups to future work.

Similarly to the awareness case characterized in Theorem \ref{thm:awareness}, we show that the solution to the fair regression problem in the unawareness framework is given by the solution to a barycenter problem with optimal transport costs. We begin  in Section \ref{sec:OT} with a brief introduction to optimal transport theory and to the main tools used in the proofs of our results. In Section \ref{sec:barycenter}, we characterize the optimal fair regression function. First, we prove in Proposition \ref{prp:not-envy-free} that \textit{in general, the optimal fair regression function $f^*$ does not preserve order}. Next, we demonstrate the following result, which relates fair regression in the unawareness framework to an optimal transport problem.

\begin{theorem}[Informal\footnotemark{}]\label{thm:informal_OT}
\footnotetext{This result is formalized in Theorem \ref{thm:ot_predictor}.}
    Under mild assumptions, the optimal fair regression function $f^*$ is given by the solution to a barycenter problem with optimal transport costs. In particular, there exists a function $\boldsymbol{f}^*$ such that 
    \begin{align*}
        f^*(x) = \boldsymbol{f}^*(\eta(x), \Delta(x)),
    \end{align*}
    where $\Delta(x) \propto \frac{\mathbb{P}(S= 1\vert X=x)}{p_1} - \frac{\mathbb{P}(S= 2\vert X=x)}{p_2}$.
\end{theorem}
Comparing this result to the one provided in Theorem \ref{thm:awareness} within the awareness framework, we note that there is no explicit formula for the optimal fair regression function within the unawareness framework. Moreover, Theorem \ref{thm:informal_OT} underscores that the optimal fair regression function effectively relies on an estimate $\Delta(X)$ of the unobserved sensitive attribute $S$ to make predictions, thereby indirectly implementing disparate treatment. This result provides a theoretical explanation for the empirical phenomenon observed by \cite{liptonDisparity}. As noted in their work,  this behavior is problematic as it can lead to basing predictions on factors that are not relevant to predicting the outcome $Y$, simply because they are informative for predicting the sensitive attribute $S$. 

In Section \ref{sec:nested}, we investigate the relationship between fair regression and fair classification when $Y \in \{0,1\}$. We demonstrate the existence of a dichotomy based on a \textit{nestedness} criterion.  Recall that as the threshold $y$ increases, the Bayes classifier $\gbayes_y$ predicts $1$ for a decreasing proportion of candidates; we show that this also holds for the optimal fair classifier $g_y^*$. We say that the fair classification problem is \textit{nested} if, almost surely with respect to the measure $\mu$ of $X$, the prediction $g^*_y(X)$ for the candidate $X$ decreases as $y$ increases. In other words, candidates rejected (i.e., with prediction $0$) at low values of $y$ cannot be accepted at higher values of $y$, when the proportion of accepted candidates is lower. For example, the Bayes classifier defined by $\gbayes_y(x) = \mathds{1}\left\{\fbayes(x)  \geq y \right\}$ satisfies this condition. When the nestedness criterion holds, the decision boundaries for the optimal fair classifier for different risk $\cR_y$ form nested sets. 
The following informal result summarizes our findings.
\begin{theorem}[Informal\footnotemark{}]\label{thm:informal_eq_classif}
    Under mild assumptions, if the fair classification problem is nested, then the regression function
    \begin{align*}
        f^*(x) = \sup\left\{y \in \R\ :\ g^*_y(x) = 1\right\}
    \end{align*}
    is optimal for the fair regression problem \eqref{eq:unaware_regression}; equivalently, the classifier
    \begin{align*}
        g_y(x) = \mathds{1}\left\{f^*(x) \geq y\right\}
    \end{align*}
    is optimal for the fair classification problem \eqref{eq:unaware_classif} for the risk $\cR_y$. Conversely, if the classification problem is not nested and if $f^*$ is the optimal fair regression function, then there exists $y \in (0,1)$ such that 
    \begin{align*}
        g_y(x) = \mathds{1}\left\{f^*(x) \geq y\right\}
    \end{align*}
    is sub-optimal for the fair classification problem with risk $\cR_y$.
\end{theorem}
\footnotetext{This result is formalized in Proposition \ref{prp:not_opt} and in Corollary \ref{cor:eq_classif}}
While nestedness may initially appear to be a natural assumption, it does not always hold. In Section \ref{sec:examples}, we show how to design examples of problems where this assumption is either met or violated.

\section{A short introduction to optimal transport}\label{sec:OT}
In this section, we provide a brief introduction to optimal transport. We present the main tools that will be used in the proofs of the theorems in Sections \ref{sec:barycenter} and \ref{sec:nested}. We begin by providing an overview of  optimal transport in Section \ref{subsec:OT}, before discussing the multi-to-one dimensional transport problem in Section \ref{subsec:MTO-OT}.

\subsection{The optimal transport problem}\label{subsec:OT}
Optimal transport provides a mathematical framework  to compare probability distributions. Consider a Borel probability measure $\mu$ on a Polish space $\cX$ and a Borel probability measure $\nu$ on some other Polish space $\cY$. We are given a  continuous cost function $c:\cX\times \cY\to [0,+\infty]$, where $c(x,y)$ represents the cost of moving a unit of mass from $x\in \cX$ to $y\in \cY$. The optimal transport problem consists in finding the optimal way of moving the distribution of mass $\mu$ to $\nu$ by minimizing the total displacement cost. Formally, a transport map is a measurable map $T:\cX\to \cY$ such that the pushforward measure $T\sharp \mu$ of $\mu$ by $T$ is equal to $\nu$, where the pushforward measure is defined for all measurable sets $B\subset \cY$ by
\begin{equation*}
    T\sharp \mu(B)=\mu(T^{-1}(B)).
\end{equation*}
The optimal transport problem  is then the following
\begin{equation}\label{eq:cost_map}
\begin{split}
 & \text{minimize}  \quad \int c(x,T(x))\dd \mu(x)\\
  &\text{under the constraint $T\sharp \mu = \nu$.} 
\end{split}
\end{equation}
The existence of minimizers of the optimization problem \eqref{eq:cost_map} is a delicate problem that depends on both  the regularity of the cost function $c$ and the properties of $\mu$ and $\nu$. For instance, when $\cX=\cY=\R^d$  and $c(x,y)=\|x-y\|^2$, a solution exists whenever $\mu$ gives zero mass to sets of dimensions smaller than $d-1$; otherwise, a solution may not exist, see \cite[Chapter 10]{villani2009optimal}.  When $\cX=\cY=\R^d$  and $c(x,y)=\|x-y\|^2$, the corresponding minimum is known  as the (squared) Wasserstein distance between $\mu$ and $\nu$, denoted by $\cW_2^2(\mu,\nu)$.
More generally, an optimal transport map exists whenever $\mu$ gives zero mass to sets of dimensions smaller than $d-1$ and the cost function $c(x,y)=\|x-y\|^2$ is replaced by any smooth cost function $c$ satisfying the so-called \textit{twist condition}, which states that the determinant $\det(\frac{\partial^2 c}{\partial y_j\partial x_i})$ never vanishes.

The optimal transport problem also admits a relaxed version in terms of transport plans, which is often more convenient to work with. A transport plan is a probability measure $\pi$ on the product space $\cX\times \cY$ which has first marginal equal to $\mu$ and second marginal equal to $\nu$: for all measurable sets $A\subset \cX$ and $B\subset \cY$,
\begin{equation*}
    \pi(A\times \cY)=\mu(A),\quad \pi(\cX\times B)=\nu(B),
\end{equation*}
or, in probabilistic terms, if $(X,Y)\sim \pi$, then $X\sim \mu$ and $Y\sim \nu$. Informally, for $x\in \cX$, the conditional law of $Y|X=x$ describes the different locations where the mass initially at $x$ will be sent. 
The cost of a transport plan $\pi$ is given by
\begin{equation*}
    \iint c(x,y) \dd \pi(x,y).
\end{equation*}
Note that a transport map $T$ induces a transport plan by considering the law $\pi$ of $(X,T(X))$ (formally, $\pi=(\id, T)\sharp \mu$). The optimal transport cost is defined by the following minimization problem:
\begin{equation}\label{eq:cost_plan}
    \OT_c(\mu,\nu)=\min_{\pi\in \Pi(\mu,\nu)}\int c(x,y)\dd \pi(x,y),
\end{equation}
where $\Pi(\mu,\nu)$ is the set of transport plans between $\mu$ and $\nu$. Optimal transport plans always exist, whereas optimal transport maps may fail to do so. When optimal transport maps exist and the source measure $\mu$ has no atoms, the minimization problem \eqref{eq:cost_map} gives the same value as the optimal transport cost defined in \eqref{eq:cost_plan}, see \cite{pratelli2007equality}.

Our proofs will rely heavily on the dual formulation of the optimal transport problem, which we now introduce. The $c$-transform of a function $\phi:\cY\to \R\cup\{+\infty\}$ is defined as
\begin{equation*}
    \forall x\in \cX,\ \phi^c(x)=\sup_{y\in \cY}(\phi(y)-c(x,y)).
\end{equation*}
The subdifferential of $\phi$ is defined as
\begin{equation*}
    \partial_c\phi=\{(x,y)\in \cX\times \cY:\ \phi(y)-\phi^c(x)=c(x,y)\}.
\end{equation*}
For the quadratic cost, these notions are closely related to the usual notions of convexity, with $c$-transforms being analogous to the concept of convex conjugates.

Kantorovich duality  \cite[Theorem 5.10]{villani2009optimal} states that 
\begin{equation*}
    \begin{split}
    \OT_c(\mu,\nu)    &=\sup_{\phi\in L^1(\nu)}\p{\int \phi(y)\dd\nu(y)-\int \phi^c(x)\dd\mu(x)}.
    \end{split}
\end{equation*}
Moreover, under the mild assumption that there exist two  functions $a\in L^1(\mu)$ and $b\in L^1(\nu)$ such that $c(x,y)\leq a(x)+b(y)$ for all $x\in \cX$, $y\in \cY$, the previous supremum is attained by a function $\phi$, which we call a Kantorovich potential. In that case, any optimal transport $\pi$ is supported on the subdifferential of the $c$-convex function $\phi$, meaning that
\begin{equation*}
    \pi(\partial_c\phi)=1.
\end{equation*}
This last condition imposes significant constraints on the structure of optimal transport plans. For the quadratic cost, this fact is the key ingredient in proving that optimal transport plans are induced by optimal transport maps.



\subsection{Multi-to-one dimensional optimal transport}\label{subsec:MTO-OT}
In the next section, we demonstrate that the fair regression problem within the unawareness framework can be reduced to a barycenter problem of the form:
\begin{equation}\label{eq:barycenter_dim2}
     \min_{\nu\in \cP(\R)} p_1\OT_c(\mu_1,\nu)+p_2\OT_c(\mu_2,\nu),
\end{equation}
where $\mu_1$, $\mu_2$ are \textit{two-dimensional} probability measures and $c:\R^2\times \R\to [0,+\infty)$ is a cost function. This reduction raises the question of whether the solutions to the barycenter problem \eqref{eq:barycenter_dim2} can be characterized by transport maps.

Proving that the optimal transport problem $\OT_c(\mu_s,\nu)$ is solved by a transport map is nontrivial. Complications arise because the measures $\mu_s$ and $\nu$ are defined on spaces of different dimensions. Optimal transport problems involving spaces of different dimensions have not been as extensively studied and exhibit distinct properties compared to the standard case where both measures are defined on the same space, see  \citep{pass2010regularity,chiappori2016multidimensional, multi-to-one,mccann2020optimal}. For instance, the classical twist condition $\det(\frac{\partial^2 c}{\partial y_j\partial x_i})\neq 0$ does not make sense in this setting: the Hessian matrix of $c$ is not squared, so that the determinant is not even well-defined. 

 \cite{multi-to-one} focus on the optimal transport problem between a measure $\mu$ supported on a domain $\cX\subset \R^m$ (with $m>1$) and a measure $\nu$ on an interval $\cY\subset \R$ for some cost function $c:\cX\times \cY\to [0,+\infty)$. They demonstrate that an optimal transport map $T:\cX\to \cY$ between $\mu$ and $\nu$ exists under a natural condition on $(c,\mu,\nu)$ known as \textit{nestedness}. For $y\in \cY$, $k\in \R$, let
\begin{equation*}
    \cX_{\leq}(y,k)=\left\{x\in \cX:\ \partial_y c(x,y)\leq k\right\}.
\end{equation*}
Kantorovich duality implies that an optimal transport plan between $\mu$ and $\nu$ will match an interval $(-\infty,y]$ to a set $\cX_{\leq}(y,k)$, where $k=k(y)$ is a solution of the equation $\nu\left((-\infty,y]\right)=\mu\left(\cX_\leq(y,k)\right)$. 
The triplet $(c,\mu,\nu)$ is called nested if the collection of sets $(\cX_\leq(y,k(y)))_y$  increases with $y$. \cite{multi-to-one} prove that an optimal transport map $T$ between $\mu$ and $\nu$ exists when the problem is nested: informally, the monotonicity of $(\cX_\leq(y,k(y)))_y$ ensures that a given $x_0\in \cX$ belongs to the boundary of a single set $\cX_\leq(y_0,k(y_0))$, with $y_0$ being equal to $T(x_0$).

This nestedness condition will be crucial in \Cref{sec:nested}, where it will be  used to establish the equivalence between regression and classification problems. However, in \Cref{sec:barycenter}, we will be able to show the existence of optimal transport maps for the barycenter problem \eqref{eq:barycenter_dim2} (and consequently of optimal fair regression functions) without any nestedness condition. 

\section{Fair regression and the barycenter problem}\label{sec:barycenter} 
In this section, we characterize the solution to the fair regression problem. We begin by showing in Section \ref{subsec:ordering} that, under mild assumptions, the fair regression function does not preserve order. Then, in Section \ref{subsec:reduction-to-OT}, we show that the fair regression problem can be reduced to a barycenter problem with optimal transport costs. Using the tools introduced in Section \ref{sec:OT}, we prove the existence of a fair optimal prediction function and study some of its properties.

\subsection{Fair regression functions do not preserve order}\label{subsec:ordering}
Before analyzing in detail the fair regression problem in the unawareness framework, we establish a simple yet important property of fair regression functions. We begin by extending the definition of order preservation \citep{Chzhen2020a, Chzhen2020AMF} to the unawareness framework. Recall that in this case, the Bayes prediction for a candidate $x$ is given by $\eta(x)=\E[Y|X=x]$. A prediction function $f$ is said to preserve order if, for any two candidates $(x,x') \in \cX^2$ in the same group $s \in \cS$, $\eta(x) \leq \eta(x')$ implies $f(x) \leq f(x')$. This definition is formalized below.

\begin{definition}[Order preservation in regression - unawareness framework]
A prediction function $f$ preserves order if $\P\otimes\P$-almost surely,
\begin{equation*}
\left \{\eta(X)<\eta(X') \text{ and } S = S' \right\}\implies f(X)< f(X').
\end{equation*}  
\end{definition}
This property implies that the fairness correction does not alter the ordering of the predictions of the Bayes prediction function within a group. It is related to the concept of ``rational ordering'' introduced by \cite{liptonDisparity} in the context of classification, where the authors require that within a group, the most able candidates are the ones accepted. 

\begin{proposition}\label{prp:not-envy-free}
    Let $f:\cX\to\R$ be a  regression function with $\E[f(X)^2]< \infty$ satisfying the demographic parity constraint. Assume that the Bayes regression function $\eta$ does not satisfy the demographic parity constraint and that 
    $\P(S=s|X=x)\in (0,1)$ for all $s\in \cS$, $x\in \cX$. Then, $f$ does not preserve order.
\end{proposition}


\begin{proof}
We prove the contrapositive: if $f$ is a regression function satisfying the demographic parity constraint and preserving order, then the Bayes regression function also satisfies the demographic parity constraint. For a fixed group $s\in \cS$, consider the joint law $\pi_s$ of $(\eta(X),f(X))$, where $X\sim \mu_s$. As $f$ is envy-free, the support of the measure $\pi_s$ is monotone, in the sense that
\begin{equation}\label{eq:monoton_plan}
\forall (y_1,z_1), (y_2,z_2)\in \supp(\pi_s),\ y_1< y_2\implies z_1<  z_2.
\end{equation}
According to \cite[Lemma 2.8]{santambrogio2015optimal}, this implies that $\pi_s$ is actually the optimal transport plan for the quadratic cost between the first marginal of $\pi_s$, equal to $\eta\sharp \mu_s\eqdef \nu_s$ and the second marginal of $\pi_s$, equal to $f\sharp \mu_s\eqdef\nu$ (the second measure does not depend on $s$ because of demographic parity). We claim that  strict monotonicity implies that the  transport plan $\pi_s$ takes the form of a transport map $T_s$ transporting $\nu$ towards $\nu_s$, that is $\pi_s = (T_s,\id)\sharp \nu$ (see a proof below). So, if $X\sim\mu_s$, we have $(\eta(X),f(X))=(T_s(f(X)),f(X))$ almost surely. To put it another way, we have for every $s$,
\begin{equation*}
 \eta(x) = T_s\circ f(x) \ \text{ $\mu_s$-almost everywhere.}
\end{equation*}
As $\P(S=s|X=x)>0$ for all $x\in \cX$, this equality is also satisfied $\mu$-almost everywhere. Hence, for $\mu$-almost all $x$, the quantity $T_s\circ f(x)$ does not depend on $s$ (it is equal to $\eta(x)$). This defines a function $T$ with  $\eta \sharp \mu_s = T\sharp f\sharp \mu_s= T\sharp \nu$. As this measure does not depend on $s$, this proves that $\eta$ satisfies the demographic parity constraint. 

To conclude our proof, it remains to prove our claim. Decompose $\nu$ as $\nu_1+\nu_2$ where $\nu_2$ is atomless and $\nu_1 = \sum_j p_j \delta_{z_j}$. If $f(X)=z_j$, then we have $\eta(X)=y_j$ for some value $y_j$: this value $y_j$ has to be unique, for otherwise it would contradict the monotonicity assumption. Therefore, $\nu_s$ can be written as $\nu_s = \nu_{1s}+\nu_{2s}$, where $\nu_{1s} = \sum_j p_j\delta_{y_j}$. Consider the plan $\pi_1 = \sum_j p_j \delta_{(y_j,z_j)}$. Then $\pi-\pi_1$ is a plan between $\nu_{2s}$ and $\nu_{2}$. By \cite[Lemma 2.8]{santambrogio2015optimal}, as $\nu_2$ is atomless, the monotonicity condition implies that it is induced by a transport map $\tilde T_s$. In total, we can define $T_s$ by $T_s(z_j)=y_j$ and by $T_s=\tilde T_s$ on the complementary set of the atoms. 
\end{proof}
Proposition \ref{prp:not-envy-free} implies, in particular, that in many instances, the optimal fair regression function does not preserve order. Consequently, highly qualified individuals who belong to  protected groups could potentially suffer from  fairness corrections due to the demographic parity constraint.

The condition that  $\P(S=s|X=x)\in (0,1)$ for all $s\in \cS$, $x\in \cX$ ensures that the sensitive attribute $S$ cannot be determined from the observation of $X$. When this condition is not satisfied, the distinction between the unawareness and awareness frameworks becomes blurred: if $S$ can be inferred from  $X$ alone, it becomes meaningless to differentiate  between a regression function that depends on both $X$ and $S$, and one that depends solely  on $X$. Furthermore, it is important to note that in the awareness framework, there do exist regression functions that satisfy the demographic parity constraint and preserve order, with the optimal fair regression function described in  \Cref{thm:awareness} being one such example.

\subsection{Reduction to an optimal transport problem}\label{subsec:reduction-to-OT}
In the following, we let $\cS=\{1,2\}$. 
We assume that $\mu_1\neq\mu_2$ (otherwise the Bayes regression function $\eta$ already solves the fair regression problem). We now show how to transform the fair regression problem into a barycenter problem using optimal transport costs. To do so, we first leverage a reformulation of the demographic parity constraint due to \cite{Chzhen2020AnEO}, which is based on  the Jordan decomposition of the signed measure $\mu_1 - \mu_2$. Then, we show how to rephrase the regression problem as a barycenter problem, using this new constraint. Finally, we show that, under mild assumptions, the barycenter problem admits a unique solution, which is given by a transport map.

\subsubsection{Reformulation of the demographic parity constraint} Let $|\mu_1-\mu_2|$ be the  variation of $\mu_1-\mu_2$ and define
\[
\begin{cases}
(\mu_1 - \mu_2)_+ = \frac{1}{2}(|\mu_1-\mu_2|+\mu_1-\mu_2), \\
 (\mu_1 - \mu_2)_- = \frac{1}{2}(|\mu_1-\mu_2|-\mu_1+\mu_2) 
 \end{cases}
 \]
the Jordan decomposition of $\mu_1-\mu_2$. The two measures $(\mu_1 - \mu_2)_+$ and $(\mu_1 - \mu_2)_-$ have the same mass, which we denote by $m$. We define the scaled Jordan decomposition of $\mu_1-\mu_2$ as the pair of probability measures
\[ \mu_+ = (\mu_1 - \mu_2)_+/m \ \text{ and }\ \mu_- = (\mu_1 - \mu_2)_-/m.\]
Let $\frac{\dd \mu_+}{\dd \mu}$ (resp. $\frac{\dd \mu_-}{\dd \mu}$) be the density of $\mu_+$ (resp. $\mu_-$) with respect to $\mu$ (that are defined uniquely $\mu$-almost everywhere). As $\mu_+$ and $\mu_-$ are mutually singular measures, we can always find versions of $\frac{\dd \mu_+}{\dd \mu}$ and $\frac{\dd \mu_-}{\dd \mu}$ such that the sets
\[
    \begin{cases}
        \cX_+=\{x\in \cX :\ \frac{\dd \mu_+}{\dd \mu}(x)>0\},\\
        \cX_-=\{x\in \cX :\ \frac{\dd \mu_-}{\dd \mu}(x)>0\},\\
        \cX_{=}=\cX \backslash ( \cX_+\sqcup  \cX_-).
    \end{cases}
\]
form a partition of $\cX$, with $\mu_+$ giving mass $1$ to $\cX_+$ and $\mu_-$ giving mass $1$ to $\cX_-$. Then, for for any three functions $f_+$, $f_-$, and $f_=$ from $\cX$ to $\mathbb{R}$, we can define the associated function $\cF(f_+, f_-, f_=)$ equal to $f_+$ on $\cX_+$, $f_-$ on $\cX_-$, and $f_=$ on $\cX_=$:
\begin{align*}
    \cF(f_+, f_-, f_=)(x) = \begin{cases}
        f_+(x) &\text{if } x \in \cX_+\\
        f_-(x) &\text{if }  x \in \cX_-\\
        f_=(x) &\text{if }  x \in \cX_=.
    \end{cases}
\end{align*}
Conversely, for any function $f:\cX \rightarrow\mathbb{R}$, there exist functions $f_+$, $f_-$, and $f_=$ corresponding respectively to the restriction of $f$ on $\cX_+$, $\cX_-$, and $\cX_=$, i.e., such that $f = \cF(f_+, f_-, f_=)$.
The following lemma, due to \cite{Chzhen2020AnEO}, rephrases the demographic parity constraint in terms of $\mu_+$ and $\mu_-$.
\begin{lemma}\label{lem:EvNico}
    A regression function $f : \cX \rightarrow \mathbb{R}$ verifies the demographic parity constraint if and only if 
    $$f\sharp \mu_+ = f\sharp \mu_-.$$
\end{lemma}


Lemma \ref{lem:EvNico} reveals that for any 
functions $f$, and $f_+$, $f_-$,  $f_=$ such that $f = \cF(f_+, f_-, f_=)$, the regression function $f$
satisfies the demography parity constraint if and only if $f_+\sharp \mu_+ =f_-\sharp \mu_-$. The two functions $f_+$ and $f_-$ can be chosen with disjoint support (in $\cX_+$ and $\cX_-$, respectively). Thus, the demographic parity constraint essentially reduces to the equality of the pushforward measures of two distinct probabilities ($\mu_+$ and $\mu_-$) by two distinct functions ($f_+$ and $f_-$). 

\subsubsection{A barycenter problem}
In order to rephrase the regression problem as a barycenter problem, we introduce further notation. We define 
\begin{equation}
    \begin{cases}
    \Delta(x)= \frac{\dd\mu_+}{\dd\mu}(x) &\text{ if } x\in \cX_+,\\
        \Delta(x)= -\frac{\dd\mu_-}{\dd\mu}(x) &\text{ if } x\in \cX_-,\\
    \Delta(x)= 0 &\text{ if } x\in \cX_=.
    \end{cases}
\end{equation}
Equivalently, $\Delta(x)$ is proportional to $\frac{\dd \mu_1}{\dd\mu}(x)-\frac{\dd \mu_2}{\dd\mu}(x)$.
We also define the cost $c : \cX \times \R \to [0,+\infty]$ given by $c(x,y) = \frac{(\eta(x) - y)^2}{|\Delta(x)|}$ for all $x\in \cX$ and all $y\in \mathbb{R}$. When $X\sim \mu_{\pm}$, the variables $(\eta(X), \Delta(X))$ belong to $\Omega \defeq \{(h,d)\in\R^2:\ d\neq 0\}.$ 
In the following, we use bold notation to denote functions related to these two-dimensional variables. For example, we denote by $\mup$ (resp. $\mum$) the distributions of $(\eta(X), \Delta(X))$ when $X$ follows the distribution $\mu_+$ (resp. $ \mu_-$). Note that the support of $\mup$ is included in the upper half-plane $\{d>0\}$ while $\mum$ is included in the lower half-plane $\{d<0\}$. 
We define the two-to-one dimensional cost $\bc$, given by $\bc(\x, y) = \frac{(h - y)^2}{|d|}$ for all $\x = (h,d)\in\Omega$ and $y\in \R$. 

Consider the barycenter problem
\begin{equation}\label{eq:barycenter_problem}
      \text{minimize}\quad \OT_{\bc}(\mup, \nu) + \OT_{\bc}(\mum,\nu) \text{ over } \nu\in\cP(\R)  
\end{equation}
where we recall that $\OT_{\bc}(\mupm, \nu)$ is the optimal transport cost for sending $\mupm$ to $\nu$ with cost function $\bc$, defined in Equation \eqref{eq:cost_plan}. We say that a solution $\nu^{\mathrm{bar}}$ of the barycenter problem is solved by optimal transport maps if 
\[ \OT_{\bc}(\mupm, \nu^{\mathrm{bar}}) = \int \bc(\x, \fpm(\x))\dd \mupm(\x)\]
for some transport maps  $\fpm:\Omega\to \R$ from $\mupm$ to $\nu^{\mathrm{bar}}$.

\begin{lemma}\label{lem:reducOT}
There is a one-to-one correspondence between the set of solutions to the barycenter problem \eqref{eq:barycenter_problem} solved by optimal transport maps and the set of optimal fair regression functions solving \eqref{eq:unaware_regression}. This correspondence associates  a barycenter $\nu^{\mathrm{bar}}$ with optimal transport maps $\fpm$ to the optimal fair regression function $f=\cF(f_+,f_-,\eta)$, where $f_{\pm}(x)= \fpm(\eta(x),\Delta(x))$ for $x\in \cX$.
\end{lemma}
\begin{proof}
Classical computations show that $\cR_{sq}(f) = \mathbb{E}\left[(\eta(X) - f(X))^2\right] + \mathbb{E}\left[(\eta(X) - Y)^2\right].$ Thus, minimizing the risk is equivalent to minimizing $\mathbb{E}\left[(\eta(X) - f(X))^2\right]$. Now,
\begin{equation}\label{eq:reformulation} 
    \begin{split}
        &\mathbb{E}\left[(\eta(X) - f(X))^2\right] = \int (\eta(x)-f(x))^2 \dd\mu(x)\\
&\quad =\int_{\mathcal{X}_+}  (\eta(x)-f(x))^2 \frac{\dd\mu}{\dd \mu_+}(x) \dd \mu_+(x) + \int_{\mathcal{X}_-} (\eta(x)-f(x))^2 \frac{\dd\mu}{\dd \mu_-}(x) \dd \mu_-(x)\\
&\quad\qquad+  \int_{\mathcal{X}_=}  (\eta(x)-f(x))^2 \dd \mu(x).
    \end{split}
\end{equation}
Using the definition of $\Delta$ along with Lemma \ref{lem:EvNico}, we see that any solution to the fair regression problem can be written as $f = \cF(f_+, f_-, \eta)$, where $(f_+, f_-)$ is solution to the problem
\begin{align*}
    &\text{minimize} \quad \int_{\mathcal{X}_+}  \frac{(\eta(x) - f_+(x))^2}{|\Delta(x)|}\dd \mu_+(x) + \int_{\mathcal{X}_-}  \frac{(\eta(x) - f_-(x))^2}{|\Delta(x)|} \dd \mu_-(x) \\
    &\text{such that} \quad  f_+\sharp \mu_+ = f_-\sharp \mu_-.
\end{align*}
The triplet $(\eta(X), \Delta(X), f(X))$ defines a  coupling $\pi_{f+}$ between $\mup$ and $\nu_{f+} = f_+ \sharp \mu_+$. Likewise, we define a coupling $\pi_{f-}$ between  $\mum$ and $\nu_{f-}$. We can rewrite \eqref{eq:reformulation} as
\begin{align*}
\mathbb{E}\left[(\eta(X) - \cF(f_+, f_-, \eta)(X))^2\right] &= \int \bc(\x,y)\dd \pi_{+f}(\x,y) + \int c(\x,y)\dd \pi_{-f}(\x,y) \\
&\geq \OT_{\bc} (\mup,\nu_{f+})+\OT_{\bc}(\mum,\nu_{f-}).
\end{align*}
The constraint  $f_+\sharp \mu_+ = f_-\sharp \mu_-$ implies that $\nu_{f+} = \nu_{f-}$. Hence,
\begin{equation}\label{eq:first_ineq}
\inf_{f\text{ fair}}\mathbb{E}\left[(\eta(X) - f(X))^2\right] \geq \inf_{\nu\in \cP(\R)} \OT_{\bc} (\mup,\nu)+\OT_{\bc}(\mum,\nu).
\end{equation}

Reciprocally, assume that there exists $\nu^{\mathrm{bar}}$ solving the above barycenter problem, and that an optimal transport map between $\mup$ and $\nu^{\mathrm{bar}}$ is given by an application $\fp : \Omega\rightarrow \R$, with $(\fp) \sharp \mup=\nu^{\mathrm{bar}}$. Likewise, we assume that there exists an optimal transport map $\fm$ between $\mum$ and $\nu^{\mathrm{bar}}$. Then, $\fm\sharp \mum = \fp\sharp \mup = \nu^{\mathrm{bar}}$. Defining $f_{\pm}(x) =\fpm(\eta(x),\Delta(x))$, we have $f_-\sharp \mu_- =  f_+\sharp \mu_+ = \nu^{\mathrm{bar}}$, and so $\cF(f_+, f_-, \eta)$ is a fair regression function. Also, we have by optimality that
\begin{align*}
 \OT_{\bc}(\mup,\nu^{\mathrm{bar}})+\OT_{\bc}(\mum,\nu^{\mathrm{bar}}) &= \int c(x,f_+(x))\dd \mu_+(x) + \int c(x,f_-(x)) \dd \mu_-(x) \\
 &= \E[(\eta(X)-\cF(f_+, f_-, \eta)(X))^2].
\end{align*}
Hence, by \eqref{eq:first_ineq}, the regression function $\cF(f_+, f_-, \eta)$ is optimal. This concludes the proof of Lemma \ref{lem:reducOT}.
\end{proof}

\subsubsection{Transport maps for the barycenter problem}

The rest of this section is devoted to proving that the barycenter problem indeed admits a solution given by transport maps, which will imply that there exists a solution to the fair regression problem. We show that this holds under the following mild regularity assumption.

\begin{assumption} \label{ass:continuity}
The measures $\mup$ and $\mum$ give zero mass to graphs of functions in the sense that for  any measurable function $F:\R\backslash\{0\}\to \R$, $\mupm(\{(F(d),d):\ d\neq 0\})=0$.
\end{assumption}
By Fubini's theorem, this assumption is trivially satisfied if $\mup$ and $\mum$ have a density with respect to the Lebesgue measure. Another interesting example is given by the awareness framework, seen as a particular instance of the unawareness framework. 

\begin{remark}[Awareness as a special case of unawareness\label{rem:aware_special_case}] 
Consider a triplet of random variable $(X,S,Y)\sim \P$, where $X\in \cX$ is a feature, $S\in \{1,2\}$ is a sensitive attribute and $Y\in \R$ is a response variable of interest. Let $Z = (X,S)$ and let $\Q$ be the law of the triplet $(Z,S,Y)$. Then, there is an  equivalence between considering an aware regression function $f(X,S)$ under law $\P$ and an unaware regression function $f(Z)$ under law $\Q$. Note that $Z$ is a random variable on $\tilde \cX=\cX\times \{1,2\}$. The laws $\mu_1$ of $Z|S=1$ and $\mu_2$ of $Z|S=2$ have disjoint support. It follows that $\cX_+=\cX\times \{1\}$ with $\mu_+=\mu_1$ and $\cX_-=\cX\times\{2\}$ with $\mu_-=\mu_2$. Then, $\Delta(x)=1/p_1$ if $x\in \cX_+$ and $\Delta(x) = -1/p_2$ if $x\in \cX_-$. In particular, both measures $\mup$ and $\mum$ are supported on horizontal lines in $\Omega$. 

In that case, Assumption \ref{ass:continuity} is equivalent to the fact that $\mup$ and $\mum$ have no atoms, which is exactly equivalent to the fact that the law of $\eta(X)$ (for $X\sim \mu$) has no atoms. This assumption is often considered to be a minimal assumption to ensure the existence of optimal fair regression functions in the awareness framework. Hence, Assumption \ref{ass:continuity} constitutes a  generalization of this assumption to the unawareness framework.
\end{remark}

\begin{theorem}\label{thm:ot_predictor}
Assume that $(X,Y,S)\sim \P$ is such that $\E[Y^2]<\infty$. 
Under Assumption \ref{ass:continuity}, there is a unique  minimizer $\nu^{\mathrm{bar}}$ of the barycenter problem 
\begin{equation} \label{eq:barycentre}
    \inf_\nu \OT_{\bc} (\mup,\nu)+\OT_{\bc}(\mum,\nu).
\end{equation}
Moreover, this problem is solved by optimal transport maps $\fpm$. In particular, there exists a unique solution $f^*$ of the regression problem under the demographic parity constraint \eqref{eq:unaware_regression}, which is given by
\begin{align*}
\forall x\in \cX,\    f^*(x) = \cF\left(\fp\bigl(\eta(x), \Delta(x) \bigr), \fm\bigl(\eta(x), \Delta(x)\bigr), \eta(x)\right).
\end{align*}

\end{theorem}

\begin{proof}  
Using Lemma \ref{lem:reducOT}, it is enough to show that the barycenter problem admits a unique solution $\nu^{\mathrm{bar}}$ such that the corresponding transport problems $\OT_{\bc}(\mup, \nu^{\mathrm{bar}})$ and $\OT_{\bc}(\mum, \nu^{\mathrm{bar}})$ are solved by transport maps.

\medskip
\textbf{Step 1: reduction to a standard transport problem.} We begin by reducing the barycenter problem \eqref{eq:barycentre} to a single two-to-two dimensional optimal transport problem $\OT_{\C}(\mup, \mum)$. The multimarginal version of the barycenter problem reads
\begin{equation}\label{eq:multimarginal}
\inf_{\nu\in \cP(\R)} \OT_{\bc} (\mup,\nu)+\OT_{\bc}(\mum,\nu)= \inf_{\rho \in \Pi(\cdot, \mup,\mum)} \int ( \bc(\x_1,y)+\bc(\x_2,y)) \dd \rho(y,\x_1,\x_2),
\end{equation}
where $\Pi(\cdot, \mup,\mum)$ stands for the set of measures on $\R \times \Omega\times\Omega$ with second marginal $\mup$ and third marginal $\mum$. Indeed, if $\rho\in \Pi(\cdot, \mup,\mum)$, then its two first marginals provide a transport plan between its first marginal $\nu$ and $\mum$, while the first and last marginals provide a transport plan between $\nu$ and $\mup$. This proves that the left-hand side of Equation \eqref{eq:multimarginal} is smaller than the right-hand side. For the other inequality, consider $\nu\in \cP(\R)$, with associated optimal transport plans $\pi_+\in \Pi(\mup,\nu)$ and $\pi_-\in \Pi(\mum,\nu)$. By the gluing lemma (see, e.g., Lemma 5.5 in \cite{santambrogio2015optimal}), there exists $\rho \in \Pi(\cdot, \mup,\mum)$ such that the joint law of the first two marginals is equal to $\pi_+$, and the joint law of the first and last marginal is equal to $\pi_-$. Then, $\OT_{\bc} (\mup,\nu)+\OT_{\bc}(\mum,\nu)=\int ( \bc(\x_1,y)+\bc(\x_2,y)) \dd \rho(y,\x_1,\x_2)$, proving that the right-hand side is smaller than the left-hand side in \eqref{eq:multimarginal}. This shows the validity of \eqref{eq:multimarginal}.

Furthermore, if $\rho$ solves the right-hand side of \eqref{eq:multimarginal}, then its first marginal $\nu$ is a barycenter. Actually, by optimality, for any $(y,\x_1,\x_2)$ in the support of the optimal $\rho$, the point $y$ necessarily minimizes the function $z\mapsto \bc(\x_1,z)+\bc(\x_2,z)$. Let us compute this minimizer. For  $\x_1 = (h_1, d_1)$ and $\x_2 = (h_2, d_2)$, we have
\begin{align*}
\bc(\x_1,y) + \bc(\x_2,y) = (h_1-y)^2/|d_1| + (h_2-y)^2 /|d_2|.
\end{align*}
This function is convex in $y$. The first order condition for optimality reads
\begin{equation}\label{eq:def_m}
    (y-h_1)/|d_1| + (y-h_2)/|d_2| = 0 \ \Longleftrightarrow \ y = m(\x_1,\x_2)\defeq  \frac{h_1 /|d_1| + h_2 /|d_2|}{1/|d_1|+1/|d_2|}.
\end{equation}
Moreover, the cost $\C(\x_1,\x_2) \defeq \inf_y \bc(\x_1, y) +  \bc(\x_2, y)$ corresponding to this minimum is equal to 
\begin{equation}\label{eq:cost_C}
\C(\x_1,\x_2) = \frac{ (h_2-h_1)^2 }{|d_1|+|d_2|}.
\end{equation}
These considerations show that
\begin{equation}\label{eq:reduction_to_C}
    \inf_{\nu\in \cP(\R)} \OT_{\bc} (\mup,\nu)+\OT_{\bc}(\mum,\nu)=\OT_{\C}(\mup,\mum),
\end{equation}
and that optimal transport plans $\pi^*\in \Pi(\mup,\mum)$ are in correspondence with barycenters $\nu$ through the formula $\nu=m \sharp \pi^*$. In particular, as there exists at least one optimal transport plan, the infimum in the barycenter problem is actually a minimum. 

\medskip
\textbf{Step 2: existence of a transport map.}
Note that 
\begin{equation}
\C(\x_1,\x_2)=\frac{(h_1-h_2)^2}{|d_1|+|d_2|} \leq 2\frac{h_1^2}{|d_1|} + 2\frac{h_2^2}{|d_2|}.
\end{equation}
This quantity is integrable against $\mup\otimes \mum$. Indeed, 
\begin{align*}
    \int \frac{h_1^2}{|d_1|}\dd \mup(h_1,d_1)=\int \frac{\eta(x)^2}{\Delta(x)} \dd \mu_+(x) = \int_{\cX_+} \eta(x)^2 \dd \mu(x) \leq \E[\E[Y|X]^2]\leq \E[Y^2]<\infty.
\end{align*}
In particular, the optimal cost $\OT_{\C}(\mup, \mum)$ is finite. Hence, by Kantorovich duality (see \Cref{sec:OT}), there is a $\C$-convex function (called a Kantorovich potential) $\phi: \Omega\to \R \cup \{+\infty\}$ such that if we let 
\begin{equation}\label{eq:ccm}
\Gamma = \{(\x_1,\x_2):\ \phi(\x_1)-\phi^{\C}(\x_2)=\C(\x_1,\x_2)\}
\end{equation}
be the subdifferential of $\phi$, 
then \textit{any} optimal transport plan $\pi$ satisfies $\pi(\Gamma)=1$, see \Cref{sec:OT}. 
 We show in Appendix \ref{app:regularity} the following lemma.
 
 \begin{lemma}\label{lem:regularity}
    Let  $\phi:\Omega \to \R\cup\{+\infty\}$ be a $\C$-convex function with $\mathrm{dom}(\phi):=\{\x:\ \phi(\x)<+\infty\}$. Then, the set  of points $\x\in \mathrm{dom}(\phi)$ such that the partial derivative $\partial_h \phi(\x)$ does not exist is included in a countable union of graphs of measurable functions $F:d\in \R\backslash\{0\}\mapsto F(d)\in \R$.
\end{lemma}

Let $\Sigma$ be the countable union of graphs given by Lemma \ref{lem:regularity} for the Kantorovich potential $\phi$. According to Assumption \ref{ass:continuity}, if we let $\Omega_0=\Omega\backslash \Sigma$, then $\mup(\Omega_0)=1$. 

Let $\x_1\in \Omega_0$ and let $(\x_1,\x_2)\in \Gamma$. Consider the function $g_{\x_2}:\x\in \Omega \mapsto \phi(\x)-\C(\x,\x_2)$. As $\phi^{\C}(\x_2) =\phi(\x_1)- \C(\x_1,\x_2)$, by definition of the $\C$-transform, the function $g_{\x_2}$ attains its maximum at $\x_1$.  In particular, as $\partial_{h_1}\phi(\x_1)$ exists by assumption, we have
\[ \partial_h \phi(\x_1)= \partial_{h_1}\C(\x_1,\x_2) = \frac{2(h_1-h_2)}{|d_1|+|d_2|}.\]
This implies that
\begin{equation}
h_2 = h_1 - \frac{|d_1|+|d_2|}{2}\partial_{h_1}\phi(\x_1).
\end{equation}
Using this expression, we find that
\begin{equation}
    m(\x_1,\x_2) =  h_1 - \frac{|d_1|\partial_{h_1} \phi(\x_1)}{2}.
\end{equation}
In particular, $m(\x_1,\x_2)$ is uniquely determined by $\x_1$. This defines a  measurable map $\x_1\in \Omega_0 \mapsto \fp(\x_1)$. We extend $\fp$ on $\Omega$ by setting $\fp(\x_1)=0$ if $\x_1\in \Omega\backslash \Omega_0$. As explained in \textbf{Step 1}, for $(\x_1,\x_2)\sim \pi^*$, the law $\nu$ of $m(\x_1,\x_2)$ solves the barycenter problem $\OT_{\C}(\mup, \mum)$.  Hence,  $\nu = m\sharp \pi^*=(\id ,\fp)\sharp \mup$ is a barycenter.

\medskip
\textbf{Step 3: uniqueness of a transport map.} 
Likewise, we show the existence of a function $\fm$ such that  $\nu' =(\id , \fm)\sharp \mum$ is a barycenter. If we show that there is a unique barycenter, then $\nu=\nu'=\nu^{\mathrm{bar}}$, and the theorem is proven.

We now show uniqueness of the barycenter. Let $\nu$ be any measure that solves the barycenter problem. Let $\pi_+$ (resp. $\pi_-$) be an optimal transport plan for $\OT_\bc(\mup,\nu)$ (resp. $\OT_\bc(\mum,\nu)$). By the gluing lemma, there exists $\rho\in \Pi(\cdot, \mup, \mum)$ whose joint law of the two first marginals is equal to $\pi_+$, and whose joint law of the first and last marginal is equal to $\pi_-$. The joint distribution $\pi$ between the second and last marginal is a transport plan between $\mup$ and $\mum$. Furthermore, as $\nu$ is a barycenter and by definition of $\C$, we have 
\begin{align*}
    \OT_{\C}(\mup,\mum)&=\OT_\bc(\mup,\nu)+\OT_\bc(\mum,\nu)= \int (\bc(\x_1,y)+c(\x_2,y))\dd \rho(y,\x_1,\x_2)\\
    &\geq \int \C(\x_1,\x_2)\dd\pi(\x_1,\x_2),
\end{align*} 
so that $\pi$ is an optimal transport plan between $\mup$ and $\mum$, with  $\nu=m\sharp \pi$. But then, recall that \eqref{eq:ccm} holds for \textit{any} optimal transport plan $\pi$ (for the same potential $\phi$). Hence, by the same arguments as before, we have $\nu=(\fp){\sharp} \mup$ for the map $\fp:\x_1\mapsto  h_1 - \frac{d_1\partial_{h_1} \phi(\x_1)}{2}$ (defined $\mup$-almost everywhere). In particular, $\nu$ is uniquely determined  by $\mup$ and $\mum$ through the potential $\phi$.
\end{proof}

Theorem \ref{thm:ot_predictor} is the counterpart of Theorem \ref{thm:awareness}, established by \cite{chzhen2020fair} and \cite{gouic2020projection} within the awareness framework. Both theorems demonstrate that the optimal fair regression function solves a barycenter problem with optimal transport costs. Remark \ref{rem:aware_special_case} further indicates that Theorem \ref{thm:awareness}  generalizes Theorem 2.3 in \cite{chzhen2020fair}, as the awareness framework can be considered as a special case of the unawareness framework. However, unlike in the awareness framework, there is no explicit formulation of the optimal fair prediction function in the unawareness framework, as the corresponding barycenter problem involves multi-to-one dimensional transport costs with no explicit solutions.

Theorem \ref{thm:ot_predictor} reveals that the fair prediction $f^*(x)$ only depends on the bi-dimensional feature $(\eta(x), \Delta(x))$ of the candidate $x$. By definition, $\Delta(x)\propto  \frac{\dd\mu_1}{\dd\mu}(x) - \frac{\dd\mu_2}{\dd\mu}(x)$. Moreover, we have $\mathbb{P}(S= 1\vert X=x) = p_1\frac{\dd\mu_1}{\dd\mu}(x)$ and $\mathbb{P}(S= 2\vert X=x) = p_2\frac{\dd\mu_2}{\dd\mu}(x)$. Thus,  $\Delta(x) \propto \frac{\mathbb{P}(S= 1\vert X=x)}{p_1} - \frac{\mathbb{P}(S= 2\vert X=x)}{p_2}$. In other words, $\Delta(x)$ reflects the probability that $x$ belongs to the different groups. Hence, in the unawareness framework, the optimal fair regression function effectively relies on estimates of $S$ to make its prediction. This result provides a theoretical justification for the empirical observations of \cite{liptonDisparity}. As noted by these authors, this phenomenon may be undesirable, as it means that the predictions can rely on features not relevant to predict the response $Y$, simply because they are predictive of the group $S$.

\section{Links between classification and regression problems}\label{sec:nested}

We now turn to the study of the relationship between fair regression and fair classification problems within the unawareness framework.  When $Y\in \{0,1\}$, classical results show that the Bayes classifier $\gbayes_y$ minimizing the risk
\begin{align*}
    \cR_{y}(g) = y\cdot \mathbb{P}\left[Y = 0,  g(X) = 1\right] + (1-y)\cdot \mathbb{P}\left[Y = 1,  g(X) = 0\right].
\end{align*}
is given by $\gbayes_y(x) = \mathds{1}\left\{\fbayes(x)  \geq y \right\}$, where $\fbayes$ is the Bayes regression function minimizing $\cR_{sq}$. Similarly, recent results by \cite{pmlr-v206-gaucher23a} demonstrate that in the awareness framework, the optimal fair classifier $g_y^*$ minimizing the risk $\cR_y$ is given by $g^*_y(x,s) = \mathds{1}\left\{f^*(x,s)  \geq y \right\}$, where $f^*$ is the optimal fair regression function minimizing $\cR_{sq}$. These results can be leveraged to obtain plug-in classifiers $\hat{g}$ using estimates $\hat{f}$ of the regression function. 

Somewhat less explored is the converse relationship: given a family of optimal classifiers $(g_y)_{y\in [0,1]}$ for the risks $(\cR_y)_{y\in [0,1]}$, one could define a regression function $f$ of the form $f(x) = \sup\{y : g_y(x) = 1\}$. For example, this formulation yields the Bayes regression function when using Bayes classifiers and the optimal fair regression function when using optimal fair classifiers in the awareness framework. In both examples, this relationship may not be particularly useful since there already exists an explicit characterization of the optimal regression function. However, if this relationship were to hold in the unawareness framework, it would be significantly more valuable. Indeed, Theorem \ref{thm:ot_predictor}  rephrases the fair regression problem as a barycenter problem with optimal transport costs but does not provide an explicit solution.

\cite{pmlr-v97-agarwal19d} proposed leveraging this relationship to address the problem of fair regression using cost-sensitive classifiers. The authors demonstrate an equivalence between minimizing a discretized version of the risk $\cR_{sq}$ and minimizing the average of the cost-sensitive risks $(\cR_y(g_{f,y}))_{y \in \cZ}$ for a finite set $\cZ$, where $g_{f,y}$ is defined as $g_{f,y}(x) = \mathds{1}\{f(x) \geq y\}$. To obtain the optimal fair regression function for this discretized risk, the authors assume access to an oracle that returns the regression function $f$ such that $g_{f,y}$ minimizes the average of the risks $(\cR_y(g_{f,y}))_{y \in \cZ}$. We emphasize that minimizing the average of the risks $(\cR_y(g_{f,y}))_{y \in \cZ}$ remains an open and challenging problem. 

In contrast, to define a regression function of the form  $f(x) = \sup\{y : g_y(x) = 1\}$, one only needs to solve independent cost-sensitive classification problems. Recent results by \cite{pmlr-v206-gaucher23a} offer an explicit characterization of these classifiers. This raises the intriguing possibility of constructing the optimal fair regression function in the unawareness framework using these fair classifiers. In this section, we demonstrate that such a construction is not always possible. To do so, we begin by providing some reminders on fair classification in the unawareness framework.

\subsection{Fair classification}
In this section, we assume that $Y\in \{0,1\}$ almost surely. We consider the problem of minimizing a family of risk measures $\cR_y$ under the demographic parity constraint. We show that the optimal fair classifier for the risk $\cR_y$ is of the form $g_y^{\kappa}$ 
for some $\kappa \in \mathbb{R}$, where $g_y^{\kappa}$
is given by 
\begin{align}\label{eq:classifier}
\forall x\in \cX,\quad &\ g_y^\kappa(x)= \ones\{\eta(x) \geq y + \kappa \Delta(x)\}.
\end{align}
The following proposition extends Proposition 5.3 in \cite{pmlr-v206-gaucher23a}, and characterizes the optimal fair classifier.
\begin{proposition}\label{prp:classif}
Let $y\in \R$, and let $\kappa^*\in \R$ verify
\begin{equation*}
\mu_+\p{\eta(X) \geq y + \kappa^*\Delta(X)} = \mu_-\p{\eta(X) \geq y + \kappa^*\Delta(X)}.
\end{equation*}
Under Assumption \ref{ass:continuity}, $g_y^{\kappa^*}$ solves the fair classification problem
\begin{equation}\tag{$C_y$}
        \begin{cases}
             \text{minimize} & \quad \mathcal{R}_y(g)\\
             \text{such that} & \quad \mathbb{E}\left[g(X) \vert S = 1\right] = \mathbb{E}\left[g(X) \vert S = 2\right].
        \end{cases}
\end{equation}
Moreover, all solutions to $(C_y)$ are a.s. equal to $g_y^{\kappa^*}$ on $\cX \setminus \{x\in \cX:\ \eta(x) = y \text { and }\Delta(x) = 0\}$.
\end{proposition}
The optimal classifier is uniquely defined outside of  $\{\eta(x) = y\}$. While the set $\{\eta(x) = y \text { and }\Delta(x) \neq 0\}$ has null measure under Assumption \ref{ass:continuity}, the set $ \{\eta(x) = y \text { and }\Delta(x) = 0\}$ may have positive measure. On this set, the classifiers $\ones\{\eta(x) \geq y\}$ and $\ones\{\eta(x) > y\}$ differ, yet they are both optimal for the risk $\cR_y(g)$.

The proof of this proposition is postponed to Appendix \ref{subsec:proof_classif}. As discussed in the previous section, Assumption \ref{ass:continuity} encompasses as a special case the awareness framework. In this case, the optimal classifier presented in Proposition \ref{prp:classif} reduces to the optimal fair classifier in the awareness framework given by
\begin{equation}
\forall (x,s)\in \cX\times \{1,2\},\quad g_y^{\mathrm{aware}}(x,s)= \begin{cases}
\ones\{\eta(x) \geq y + \frac{\kappa^*}{p_1}\} & \text{ if } s=1\\
\ones\{\eta(x) \geq y - \frac{\kappa^*}{p_2}\} & \text{ if } s=2,
\end{cases}
\end{equation}
as described in \cite{pmlr-v161-schreuder21a,Zeng2022BayesOptimalCU}.

\bigskip

\noindent In the unawareness framework, the optimal classifier relies on the probability that the observation $X$ belongs to the different groups $\Delta(X)$. This behavior is similar to that of the optimal fair regression function, as established in Theorem \ref{thm:ot_predictor}. Next, we investigate whether the optimal classifier is envy-free.

\begin{definition}[Envy-free classifiers]
    We say that a classifier $g : \cX \to \{0,1\}$ is envy-free within group if $\P\otimes \P$-a.s., for $(X,S)$ and $(X',S')$ such that $S = S'$ and $\gbayes(X) > \gbayes(X')$, we have $g(X) \geq g(X')$.
\end{definition}
In essence, this property ensures that no candidate who would have been accepted by the Bayes classifier but is rejected after fairness correction envies another candidate \textit{from the same group} who would have been rejected by the Bayes classifier but is accepted after fairness correction. Note that this property is weaker than order preservation, as a classifier that preserves order is envy-free within groups. 

Proposition \ref{prp:envy-free-classif} reveals that in the unawareness framework, the optimal fair classifier is generally not envy free. This behavior contrasts with that of optimal fair classifiers in the awareness framework: indeed, since these classifiers preserve order, they are also envy-free.

\begin{proposition}\label{prp:envy-free-classif}
Let $y\in \R$. Under Assumption \ref{ass:continuity}, if $\mathbb{P}(S = s\vert X = x)\in (0,1)$ for all $s\in \cS$, $x\in \cX$, then one of the following cases hold:
\begin{enumerate}
    \item \label{case1} $\gbayes_y(x) = 1 \implies g^{\kappa^*}_y(x) = 1$ $\mu$-a.s.
    \item \label{case2} $g^{\kappa^*}_y(x) = 1 \implies \gbayes_y(x) = 1$ $\mu$-a.s.
    \item the classifier $g^{\kappa^*}_y$ is not envy-free within group.
\end{enumerate}
\end{proposition}

\begin{proof}
Assume that \ref{case1}. and \ref{case2}. do not hold. Then, we have $\kappa^* \neq 0$, and we can assume without loss of generality that $\kappa^* > 0$. Since \ref{case1}. does not hold, we have that $\mu(\gbayes_y(X) = 1 \text{ and } g^{\kappa^*}_y(X) = 0)>0$.  This implies in turn that $\mu_+(\gbayes(X) = 1 \text{ and } g^{\kappa^*}_y(X) = 0)>0$, since $\gbayes_y$ and $g^{\kappa^*}_y$ coincide on $\cX_=$, and since by definition when $\kappa^* >0$, we have $\mu_-(\gbayes_y(X) = 1 \text{ and } g^{\kappa^*}_y(X) = 0) = 0$. Similarly, we can show that since \ref{case2}. does not hold, ${\mu_-}(\gbayes_y(X) = 0 \text{ and } g^{\kappa^*}_y(X) = 1) >0$. Now, $\mathbb{P}(S = s\vert X = x)\in (0,1)$ for all $s\in \cS$, so $\mu_1 \gg \mu_+$ and $\mu_1 \gg \mu_-$. Therefore, ${\mu_1}(\gbayes_y(X) = 1 \text{ and } g^{\kappa^*}_y(X) = 0) > 0$, and ${\mu_1}(\gbayes_y(X) = 0 \text{ and } g^{\kappa^*}_y(X) = 1)> 0$. This implies $$\mathbb{P}\left(\gbayes_y(X) > \gbayes_y(X') \text{ and } g^{\kappa^*}_y(X) <  g^{\kappa^*}_y(X')\big \vert S = S' = 1\right)>0$$
which concludes the proof.
\end{proof}

\paragraph{Extending cost-sensitive classification to $\cY = \mathbb{R}$} In the following, we consider the more general case where $\cY = \mathbb{R}$. Although the interpretation in terms of optimal classification is no longer applicable, we can still analyze the family of functions $g_y^{\kappa}$ defined in Equation \eqref{eq:classifier}. The following proposition characterizes the values of the parameter $\kappa^*$ such that $g_y^{\kappa^*}$ satisfies demographic parity. These values partition the feature space equally, see \Cref{fig:def_kappa}.

\begin{figure}
    \centering
    \includegraphics[width=0.7\linewidth]{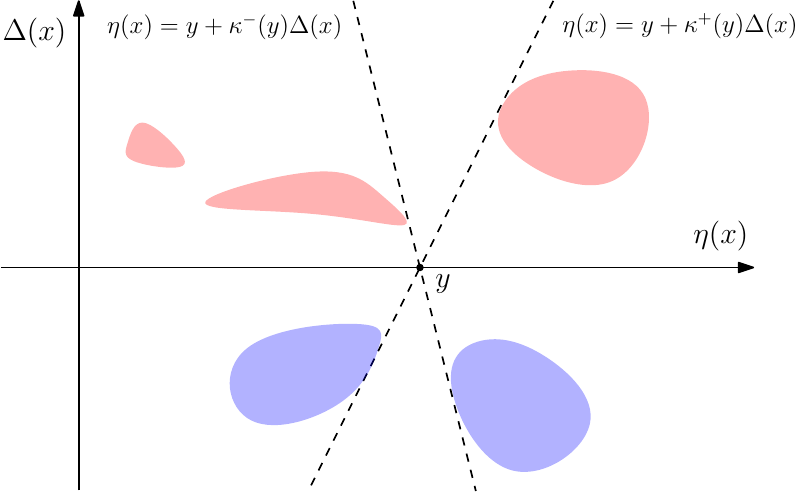}
    \caption{The measure $\mup$ is displayed in red and the measure $\mum$ is displayed in blue. By definition of $\kappa^+(y)$ and $\kappa^-(y)$,  the red region and the blue region to the right of the two dotted lines have equal masses. The region in between the two lines contains no mass.}
    \label{fig:def_kappa}
\end{figure}

\begin{proposition}\label{prp:intervalle}
Let $y\in \R$. Under Assumption \ref{ass:continuity}, the set of numbers $\kappa \in \R$ such that
\begin{equation}\label{eq:def_kappa}
\mu_+\p{\eta(X)\geq y + \kappa\Delta(X)} = \mu_-\p{\eta(X)\geq y + \kappa\Delta(X)}
\end{equation}
is a nonempty closed interval $I(y) = [\kappa^-(y),\kappa^+(y)]$. The function $y\mapsto \kappa^+(y)$ is upper semicontinuous and the function $y\mapsto \kappa^-(y)$ is lower semicontinuous.  
Moreover, it holds that for $\mu$-almost all $x$, for all $y \in \mathbb{R}$ and all $\kappa, \kappa' \in I(y)$, $g_y^{\kappa}(x) = g_y^{\kappa'}(x)$.
\end{proposition}
\begin{proof}
Introduce the function 
\[ G:(\kappa,y)\mapsto \mu_+\p{\eta(X)\geq y + \kappa\Delta(X)} -\mu_-\p{\eta(X)\geq y + \kappa\Delta(X)}.\] Under Assumption \ref{ass:continuity}, the measures $\mup$ and $\mum$ give zero mass to non-horizontal lines, implying that the function $G$ is continuous. Furthermore, for $y\in \R$, the function $G(\cdot,y)$ is nonincreasing (recall that $\Delta(X)<0$ for $X\sim \mu_-$). Hence, its zeroes form a closed interval $I(y)$. For $\kappa\in \R$, the set $\{y\in \R:\ \kappa^-(y)>\kappa\}$ is equal to the set $\{y\in \R:\ G(\kappa,y)>0\}$, which is an open set by continuity of $G$. This proves that $\kappa^-$ is lower semicontinuous. We prove similarly that $\kappa^+$ is upper semicontinuous.

It remains to prove the last statement. Fix $y\in \R$. First, we may assume without loss of generality that $\kappa=\kappa^-(y)$ and that $\kappa'=\kappa^+(y)$. 
 We have $\mu_+\p{\eta(X)\geq y + \kappa^+(y)\Delta(X)} = \mu_+\p{\eta(X)\geq y + \kappa^-(y)\Delta(X)}$ (and likewise for $\mu_-$). Thus, we have 
\begin{align*}
 \mu_+(g^\kappa_y(X)\neq g^{\kappa'}_y(X)) &= \mu_+(g^\kappa_y(X)=1,\  g^{\kappa'}_y(X)=0)  \\
 &= \mu_+\p{ \frac{\eta(X)-y}{\Delta(X)} \in [\kappa,\kappa']}=0. 
\end{align*}
The same equality holds for $\mu_-$. Also,  the equality $g^{\kappa^+(y)}_y(x)=g^{\kappa^-(y)}_y(x)$ holds on $\cX_=$ (as $\Delta(x)=0$ on $\cX_=$). Hence, for a fixed $y$, the equality $g^{\kappa^+(y)}_y(x)=g^{\kappa^-(y)}_y(x)$ holds for $\mu$-almost all $x$.

However, the set of points $x$ (of full measure) where this equality is satisfied depends on $y$, so that it is not trivial to show that this equality holds simultaneously for all $y\in \R$, almost surely.

To do so, we show that the set $\{(h,d)\in \Omega:\ \exists y\in \mathbb{R}, \ y+\kappa^-(y)d\leq  h \leq y+\kappa^+(y)d\}$ has mass 0 under $\mup$ and $\mum$. For $y\in \R$, let $C_y=\{(h,d)\in \Omega:\ y+\kappa^-(y)d\leq  h \leq y+\kappa^+(y)d\}$. We have previously shown that for any given $y\in \mathbb{R}$, $g_y^{\kappa^-(y)}(x)=g_y^{\kappa^+(y)}(x)$ for $\mu$-almost all $x$, implying  that $\mupm(C_y)=0$. Let $C=\bigcup_{y\in \R} C_y$. Let us show that $\mup(C)=0$. Let $C_1 =\bigcup_{y\in \R} \mathring C_y$. First, it holds that $\mup(C_1)=0$. 
If it were not the case, as the measure $\mup$ is inner regular, there would exist a compact set $K\subset C_1$ with $\mup(K)>0$. But then, the compact set $K$ is covered by the family of open sets $(\mathring C_y)_{y\in \R}$. By compactness, there exists a finite cover $\mathring C_{y_1},\dots,\mathring C_{y_N}$ covering $K$. As each $\mathring C_{y_i}$ has zero mass, we obtain a contradiction with the positivity of $\mup(K)$. Furthermore, if $(h,d)\in C\backslash C_1$, then there exists $y_0$ with either $h=y_0+\kappa^-(y_0)d$ or $h= y_0+\kappa^+(y_0)d$, with also $ y+\kappa^-(y)d\leq  h \leq y+\kappa^+(y)d$ for all $y\in \R$. This implies that $C\backslash C_1$ is included in the union of the graphs of the two functions $d\mapsto \sup_y (y+\kappa^-(y)d)$ and $d\mapsto \inf_y(y+\kappa^+(y) d)$. These two functions can easily be seen to be measurable because of the semicontinuity of $\kappa^-$ and $\kappa^+$. Thus, by Assumption \ref{ass:continuity}, $\mup(C\backslash C_1)=0$. In conclusion, we have proven that $\mup(C)=0$. We show likewise that $\mum(C)=0$.
\end{proof}

\subsection{The nestedness assumption} Recall that our goal is to determine whether the optimal fair regression function can be expressed as $f^*(x) = \sup\{ y : g_y^{\kappa(y)}(x) = 1\}$ for a certain  choice $\kappa(y)\in I(y)$.  In this section, we introduce an assumption regarding the family of classifiers $g_y^{\kappa(y)}$ and demonstrate that this assumption is necessary for the relationship to hold. Specifically, we wish to use the decision boundaries of optimal classifiers at different risk levels to define the regression function. For this to be possible, the function $y \mapsto g_y^{\kappa(y)}(x)$ must be nonincreasing for any choice of $x$: in other words, the rejection regions $\{x : g^{\kappa(y)}_y(x) < y\}$ must be nested. We formalize this assumption in the following definition.

\begin{definition}[Nestedness]\label{def:nested}
We say that the problem corresponding to $(X,Y,S)\sim \P$ is nested if there  exists a 
choice of $\kappa(y)\in I(y)$ for all $y\in \R$ such that 
\begin{equation}\label{eq:nested}\tag{\textbf{Nested}}
\text{for }\mu\text{-almost all } x\in \cX,\ \text{the map } y\in \R\mapsto g^{\kappa(y)}_y(x) \text{ is nonincreasing.}
\end{equation}
\end{definition}
A straightforward (but key) property implied by nestedness is the fact that the sets 
\begin{equation}
     \forall y\in \R,\ A(y) = \{x\in\cX:\ \eta(x)< y + \kappa(y)\cdot \Delta(x)\}
\end{equation}
are ``almost'' nested, in the sense that there exists a set $\tilde \cX$ of full $\mu$-measure such that for all $y'\leq y$, it holds that $A(y')\cap \tilde \cX\subseteq A(y)\cap \tilde \cX$. 

\begin{lemma}\label{lem:charac_nested_easy}
Assume that Assumption \ref{ass:continuity} holds. Then, 
the problem is nested with choice $\kappa(y)\in I(y)$ for all $y\in \R$ if and only if for all $y<y'$, 
 \begin{equation}\label{eq:charac_nested_easy}
     \mu(g_y^{\kappa(y)}(X)=0 \text{ and } g_{y'}^{\kappa(y')}(X)=1)=0.
 \end{equation} 
 Furthermore, if the problem is nested, one can always choose $\kappa(y)=\kappa^+(y)$ for all $y\in \R$ in the definition of $g_y^{\kappa(y)}$.
\end{lemma}
\begin{proof}
The direct implication is clear. For the converse one, 
    assume that the nestedness assumption does not hold. By definition, there exists a measurable set $\cX_0$ of positive $\mu$-mass such that $y\mapsto g_y^{\kappa(y)}(x)$ is not nonincreasing for all $x\in \cX_0$. It holds that either $\mu_+(\cX_0)>0$ or that $\mu_-(\cX_0)>0$. Assume without loss of generality that the first condition is satisfied, and let $\tilde \cX$ be the set of points $x$ in $\cX_0\cap \cX_+$ that satisfy 
\begin{equation}\label{eq:it_is_in_C}
    \forall y\in \R,\ \frac{\eta(x)-y}{\Delta(x)}\not \in [\kappa^-(y),\kappa^+(y)]
\end{equation}
According to Proposition \ref{prp:intervalle}, $\mu_+(\tilde \cX)=\mu_+(\cX_0\cap\cX_+)>0$. 
    For $x \in \tilde \cX$, there exists $y<y'$ with $g_y^{\kappa(y)}(x) = 0$ and $g_{y'}^{\kappa(y')}(x) = 1$. As $x$ satisfies \eqref{eq:it_is_in_C}, we have that 
    \begin{equation}\label{eq:conditions_for_non_nestedness}
        \begin{cases}
\eta(x)< y+ \kappa^-(y)\Delta(x)\\
\eta(x)> y'+\kappa^+(y')\Delta(x).
\end{cases}
    \end{equation}
Because the function $\kappa^-$ is lower semicontinuous, for $\tilde y$ close enough to $y$, we also have $\eta(x)<\tilde y+\kappa^-(\tilde y)\Delta(x)$. Likewise, there exists $\tilde y'$ close enough to $y'$ with $\eta(x)>\tilde y'+\kappa^+(\tilde y')\Delta(x)$. In conclusion, we have shown that
\begin{equation}
    \tilde \cX \subset \bigcup_{\substack{y,y'\in \Q\\ y<y'}} \{x:\ \eta(x)< y+ \kappa^-(y)\Delta(x) \text{ and }\eta(x)> y'+\kappa^+(y')\Delta(x)\}
\end{equation}
In particular, as $\mu_+(\tilde \cX)>0$, there exists $y<y'\in \Q$ with 
\[\mu_+(\eta(X)< y+ \kappa^-(y)\Delta(X) \text{ and }\eta(X)> y'+\kappa^+(y')\Delta(X))>0.\]  According to Proposition \ref{prp:intervalle} and Assumption \ref{ass:continuity}, as the equality $ \eta(X)= y'+\kappa^+(y')\Delta(X)$ happens with zero $\mu_+$-probability, we have 
    \[ \mu_+(g_y^{\kappa(y)}(X)=0 \text{ and } g_{y'}^{\kappa(y')}(X)=1)>0,\]
proving the first claim.

The second claim follows from the characterization of nestedness that we have just established. Indeed, let $y<y'$. By Proposition \ref{prp:intervalle}, for $\mu$-almost all $x$, $g_y^{\kappa(y)}(x)=g_y^{\kappa^+(y)}(x)$ and $g_{y'}^{\kappa(y')}(x)=g_{y'}^{\kappa^+(y')}(x)$. Thus, if \eqref{eq:charac_nested_easy} holds for $\kappa(y)$ and $\kappa(y')$, it also holds for $\kappa^+(y)$ and $\kappa^+(y')$.
\end{proof}

As a warm-up, we begin by showing that the nestedness assumption is always verified in the awareness setting.

\begin{proposition}\label{prop:nested_awareness}
    Assume that $S$ is $X$-measurable. Then, under Assumption \ref{ass:continuity}, the classification problem is nested.
\end{proposition}

\begin{proof}
We prove Proposition \ref{prop:nested_awareness} by contradiction. Assume that the problem is not nested. Using Lemma \ref{lem:charac_nested_easy}, there exist $y < y'$ and a set $\cX_0$ of positive $\mu$ probability such that for all $x\in \cX_0$, $g_y^{\kappa(y)}(x) = 0$ and $g_{y'}^{\kappa(y')}(x) = 1$. Using Proposition \ref{prp:intervalle}, we can also assume without loss of generality that $\eta(x) < y + \kappa^-(y)\Delta(x)$ and $\eta(x) > y' + \kappa^+(y')\Delta(x)$ for $x\in \cX_0$.
Letting $x\in \cX_0$, that we assume without loss of generality is in $ \cX_+$, the previous inequalities become
\begin{equation*}
   y' + \frac{\kappa^+(y')}{p_1} < \eta(x) < y + \frac{\kappa^-(y)}{p_1}.  
\end{equation*}
In words, the threshold for admission is lower at level $y'$ that at level $y$. This implies in particular that $\mu_+(g_{y'}^{\kappa^+(y')}(X) = 1) \geq \mu_+(g_{y}^{\kappa^-(y)}(X) = 1)$. On the other hand, since $y<y'$, it also implies that $\kappa^+(y')<\kappa^-(y)$. Therefore,  $y - \frac{\kappa^-(y)}{p_2} < y' - \frac{\kappa^+(y')}{p_2} $, so $\mu_-(g_{y'}^{\kappa^+(y')}(X) = 1) \leq \mu_-(g_{y}^{\kappa^-(y)}(X) = 1)$. Using $\mu_+(g_{y'}^{\kappa^+(y')}(X) = 1) = \mu_-(g_{y'}^{\kappa^+(y')}(X) = 1)$ and $\mu_+(g_{y}^{\kappa^-(y)}(X) = 1) = \mu_-(g_{y}^{\kappa^-(y)}(X) = 1)$, we find that $\mu_+(g_{y'}^{\kappa^+(y')}(X) = 1) = \mu_+(g_{y}^{\kappa^-(y)}(X) = 1)$. It implies that
\begin{equation*}
    \mu_+\left(\eta(X) \in \left[y' + \frac{\kappa^+(y')}{p_1}, y + \frac{\kappa^-(y)}{p_1}\right]\right) = 0.
\end{equation*}
Likewise, 
\begin{equation*}
    \mu_-\left(\eta(X) \in \left[y - \frac{\kappa^-(y)}{p_2}, y' - \frac{\kappa^+(y')}{p_2} \right]\right) = 0.
\end{equation*}
In particular, for $\kappa = \kappa^+(y') + p_1(y' - y)$, we see that $y + \frac{\kappa}{p_1} = y' + \frac{\kappa^+(y')}{p_1}< y + \frac{\kappa^-(y)}{p_1}$. Thus, $\kappa < \kappa^-(y)$, and $y - \frac{\kappa}{p_2} \geq y - \frac{\kappa^-(y)}{p_2}$. Moreover, $\kappa > \kappa^+(y')$, and $y < y'$, so $y - \frac{\kappa}{p_2} < y' - \frac{\kappa^+(y')}{p_2}$. This implies that 
\begin{align*}
    y' - \frac{\kappa^+(y')}{p_2} - (y - \frac{\kappa}{p_2}) = y' - y + \frac{1}{p_2}\left(\kappa^+(y') + p_1(y' - y) - \kappa^+(y') \right) > 0.
\end{align*}
so $y - \frac{\kappa}{p_2} \in \left[y - \frac{\kappa^-(y)}{p_2}, y' - \frac{\kappa^+(y')}{p_2} \right]$. Thus, 
\begin{equation*}
    \mu_+\left(\eta(X) \geq y + \frac{\kappa}{p_1}\right) = \mu_-\left(\eta(X) \geq y - \frac{\kappa}{p_2}\right)
\end{equation*}
and $\kappa \in I(y)$. Since $\kappa < \kappa^-(y)$, this yields a contradiction.
\end{proof}
Somewhat surprisingly, although the nestedness assumption may  appear  intuitive, it is not always verified. In Section \ref{sec:examples}, we present examples where this assumption holds and others where it does not. 

Before proving  in the next section that under the nestedness assumption, the optimal fair classification functions $g_y^{\kappa(y)}$ can be recovered by thresholding the optimal fair regression function $f^*$, we prove the converse: if the problem is not nested, there exists a value of $y\in \R$ where the classifier $\ones\{f^*(x)\geq y\}$ is suboptimal for the fair classification problem $(C_y)$.


\begin{proposition}\label{prp:not_opt}
Assume that $\cY = \{0,1\}$, that 
Assumption \ref{ass:continuity} holds and that the classification problem is not nested. Let $f^*$ be the optimal fair regression function in the unawareness framework. Then, there exists $y\in \R$ such that the classifier $x\mapsto\ones\{f^*(x)\geq y\}$ is not the optimal fair classifier for the risk $\cR_y$.
\end{proposition}

\begin{proof}
According to Lemma \ref{lem:charac_nested_easy}, there exists $y<y'$ with 
\[ \mu(g_y^{\kappa(y)}(X)=0 \text{ and } g_{y'}^{\kappa(y')}(X)=1)>0.\]
    Let $ \cX_0$ be the set corresponding to this event. 
Let us consider a classifier of the form $g_y(x) = \ones\{f(x) \geq y\}$. On the one hand, if $\mu\left(X\in  \cX_0 \text{ and } f(X) < y \right) >0$, then the probability $\mu\left(X\in  \cX_0 \text{ and } f(X) < y' \right)$ is also positive, so $g_{y'}(x)$ and $g_{y'}^{\kappa(y')}(x)$ disagree on a set of positive probability. Now, Proposition \ref{prp:classif} implies that all optimal classifiers are a.s. equal, so $g_{y'}$ is sub-optimal. On the other hand, if $\mu\left(X\in  \cX_0 \text{ and } f(X) < y \right) = 0$, then $g_y(x)=1$ for almost all $x\in  \cX_0$. This implies that $g_y(x)$ and $g_y^{\kappa(y)}(x)$ disagree on a set of positive probability, so $g_y$ is sub-optimal.
\end{proof}

\subsection{Constructing a regression function using nested classifiers}

In the previous section, we proved that under mild assumptions, nestedness is a necessary condition for the relationship $g^*_y(x) = \ones \{f^*(x) \geq y\}$ between the optimal fair classification and regression functions to hold. We now conclude by showing that nestedness is also a sufficient condition for this relationship to hold.

We begin by defining the function $f^* : \cX \rightarrow \R$ as
\begin{equation}
    \forall x \in \cX,\ f^*(x) = \sup\{y : g_y^{\kappa(y)}(x) = 1\}
\end{equation}
where $g_y^{\kappa(y)}$ is given by Equation \eqref{eq:classifier}. We assume without loss of generality (using Lemma \ref{lem:charac_nested_easy}) that $\kappa(y)=\kappa^+(y)$ for all $y\in \R$. Remark that $f^*$ is then almost measurable because of the upper semicontinuity of $\kappa^+$, in the sense that its restriction to some set of full measure is measurable (here given by the set of full measure where $y\mapsto g_y^{\kappa(y)}(x)$ is nonincreasing).

\begin{theorem}\label{thm:eq_classif}
Assume that the classification problem is nested and that Assumption \ref{ass:continuity} is satisfied. Then, the regression function $f^*$ is optimal for the fair regression problem \eqref{eq:unaware_regression}.
\end{theorem}
Before proving Theorem \ref{thm:eq_classif}, we state the following corollary.
\begin{corollary}\label{cor:eq_classif}
Assume that the classification problem is nested, that Assumption \ref{ass:continuity} is satisfied, and that $\cY = \{0,1\}$. Then, the classification function $g_y : y \mapsto \ones \{f^*(x) \geq y\}$ is optimal for the fair classification problem problem with cost $\cR_y$, where $f^*$ is the solution to the fair regression problem \eqref{eq:unaware_regression}.
\end{corollary}
The proof of Corollary \ref{cor:eq_classif} follows immediately by noticing that by Theorem \ref{thm:ot_predictor}, $f^*$ is uniquely defined, and that the nestedness assumption and Theorem \ref{thm:eq_classif} imply that $g_y(x) = g_y^{\kappa(y)}(x)$ a.s.

\medskip

The rest of the section is devoted to proving Theorem \ref{thm:eq_classif}. To do so, we begin by proving that $f^*$ is a fair regression function, and by defining $F$, the c.d.f. of the predictions under $\mu_+$ and $\mu_-$.

\begin{lemma}\label{lem:def_F}
Assume that the problem is nested  and that Assumption \ref{ass:continuity} is satisfied.  Let $F:\R\to \R$ be defined by 
\begin{equation}
\forall y\in \R,\ F(y) =  \mu_+\p{\eta(X)\leq y+\kappa(y)\Delta(X)}= \mu_-\p{\eta(X)\leq y+\kappa(y)\Delta(X)}.
\end{equation}
Then, there exists a probability measure $\nu^*$ with continuous c.d.f. $F$ and finite second moment such that $f^*\sharp \mu_+ = f^*\sharp \mu_-= \nu^*$. In particular, $f^*$ is a fair regression function.
\end{lemma}
\begin{proof}
The ``almost'' nestedness of the sets $(A(y))_y$ implies that $F$ is nondecreasing. Let us show that the function $F$ is the c.d.f. of some continuous random variable, i.e., that it goes to $0$ in $-\infty$, that it goes to $1$ in $+\infty$, and that it is continuous.

First, recall that $\Delta(X)>0$ for $X\sim \mu_+$ and that $\Delta(X)<0$ for $X\sim \mu_-$. Thus, if $\kappa(y)\leq 0$, then $F(y)\leq  \mu_+\p{\eta(X)-y \leq 0}$ and if $\kappa(y)\geq 0$, then $F(y)\leq  \mu_-\p{\eta(X)-y \leq 0}$. Hence, 
\begin{equation}\label{eq:naive_bound_F}
           F(y)\leq \max\left\{ \mu_+\p{\eta(X)-y \leq 0},  \mu_-\p{\eta(X)-y\leq 0}\right\}, 
\end{equation}
and $F$ converges to $0$ in $-\infty$. Similarly, $F(y)\to 1$ when $y$ converges to $+\infty$.

Next, let us show that $F$ is continuous. Let $y_0, y_1\in \R$ be such that $y_0 < y_1$. Now, if $\kappa(y_0)\geq \kappa(y_1)$, then 
\begin{align*}
 F(y_1) - F(y_0) &= \mu_+\p{\eta(X)\leq y_1+\kappa(y_1)\Delta(X)} - \mu_+\p{\eta(X)\leq y_0+\kappa(y_0)\Delta(X)}\\
&\leq  \mu_+\p{\eta(X)\leq y_1+\kappa(y_0)\Delta(X)} - \mu_+\p{\eta(X)\leq y_0+\kappa(y_0)\Delta(X)}
\end{align*}
Similarly, if $\kappa(y_0)\leq \kappa(y_1)$, then, (recalling that $\Delta(X)<0$ for $X\sim \mu_-$)
\begin{align*}
F(y_1) - F(y_0) &=\mu_-\p{\eta(X)\leq y_1+\kappa(y_1) \Delta(X) } -  \mu_-\p{\eta(X)\leq y_0+\kappa(y_0) \Delta(X) } \\
&\leq  \mu_-\p{\eta(X)\leq y_1+\kappa(y_0) \Delta(X) } -  \mu_-\p{\eta(X)\leq y_0+\kappa(y_0) \Delta(X) }.   
\end{align*}
Thus, 
\begin{align*}
F(y_1) - F(y_0) \leq \mu_+\p{\eta(X) - \kappa(y_0)\Delta(X) \in \left[y_0, y_1\right]} + \mu_-\p{\eta(X) - \kappa(y_0)\Delta(X)  \in \left[y_0, y_1\right]}
\end{align*}
We have shown that $F$ is non-decreasing, so $F(y_1) - F(y_0) \geq 0$. Under Assumption \ref{ass:continuity}, $\mu_+$ and $\mu_-$ give zero mass to the sets $\{\eta(X) = y_0 + \kappa(y_0)\Delta(X)\}$, so $F(y_1) - F(y_0) \rightarrow 0$ as $y_1 \rightarrow y_0^+$. This proves that $F$ is right-continuous. To show that $F$ is left-continuous, we note that if $\kappa(y_0)\geq \kappa(y_1)$, then 
\begin{align*}
 F(y_1) - F(y_0) &\leq \mu_+\p{\eta(X)\leq y_1+\kappa(y_1)\Delta(X)} - \mu_+\p{\eta(X)\leq y_0+\kappa(y_1)\Delta(X)}.
\end{align*}
Similarly, if $\kappa(y_0)\leq \kappa(y_1)$, then, 
\begin{align*}
F(y_1) - F(y_0) &\leq  \mu_-\p{\eta(X)\leq y_1+\kappa(y_1) \Delta(X) } -  \mu_-\p{\eta(X)\leq y_0+\kappa(y_1) \Delta(X) }.   
\end{align*}
Thus, 
\begin{align*}
F(y_1) - F(y_0) \leq \mu_+\p{\eta(X) - \kappa(y_1)\Delta(X) \in \left[y_0, y_1\right]} + \mu_-\p{\eta(X) - \kappa(y_1)\Delta(X)  \in \left[y_0, y_1\right]}
\end{align*}
and $F$ is also left-continuous.

Then, let us show that $\nu^*$ has finite second moment. Let $Z\sim \nu^*$. We have
\begin{align*}
     \E[Y^2] = \int_0^{+\infty} \P(Z^2>t) \dd t= \int_0^{+\infty} (F(\sqrt{t})+(1-F(-\sqrt{t})))\dd t
\end{align*}
We use \eqref{eq:naive_bound_F} to obtain that for $y\in \R$, $F(y)\leq \max(\mu_+(\eta(X)\leq y), \mu_-(\eta(X)\leq y))$. But, as $\E[Y^2]<+\infty$, the random variable $\eta(X)$ has a finite second moment under the law of either $\mu_+$ or $\mu_-$. In particular, $ \int_0^{+\infty} F(\sqrt{t}) \dd t$ is finite. Similarly, $\int_0^{+\infty} (1-F(-\sqrt{t}))\dd t$ is finite.
    
Finally, we prove the statement $f^*\sharp \mu_+ = \nu^*$. Indeed, for all $y_0 \in \R$, we have using that upper semicontinuity of $\kappa(y)=\kappa^+(y)$ that
\begin{align*}
  f^*\sharp \mu_+((-\infty,y_0])&= \mu_+ \p{ \sup \{y : g_y^{\kappa(y)}(X) = 1\}\leq  y_0 }\\
    &= \mu_+ \p{ g_{y_0}^{\kappa(y_0)}(X) = 0}=\mu_+(\eta(X)<y_0+\kappa(y_0)\Delta(X))=F(y_0).
\end{align*}
where the second line follows from the nestedness assumption and the fact that the line $\{\eta(X)=y_0+\kappa(y_0)\Delta(X)\}$ has zero mass. We show similarly that $f^*\sharp \mu_- = \nu^*$, thus concluding the proof of Lemma \ref{lem:def_F}. 
\end{proof}

The function $f^*$ depends only on $x$ through the pair $(\eta(x),\Delta(x))$. Let $\f^*:\Omega\to \R$ be defined by the relation $f^*(x) = \f^*(\eta(x),\Delta(x))$ for $x\in \cX_{\pm}$. 
We show that $\f^*$ defines an optimal transport map between $\mup$ and $\nu^*$ with respect to the cost $\bc$.

\begin{lemma}\label{prop:potential}
Assume that the problem is nested. Then, $\f^*$ is an optimal transport map between $\mup$ and $\nu^*$ for the cost $\bc$, with Kantorovich potential between $\nu^*$ and $\mup$ given by $v:y\mapsto -2\int_0^y \kappa(t)\dd t$. The same holds for $\mum$, with Kantorovich potential given by $-v$.
\end{lemma}

The proof of Lemma \ref{prop:potential} relies on the following technical lemma, whose proof is postponned to Appendix \ref{proof:kappa}.
\begin{lemma}\label{lem:tail_kappa}
    The function $y\mapsto \kappa(y)$ satisfies $|\kappa(y)|\leq C(1+|y|)$ for some $C>0$.
\end{lemma}
\begin{proof}
To prove Lemma \ref{prop:potential}, we begin by remarking that the potential $v$ is in $L^1(\nu^*)$ because of Lemma \ref{lem:def_F} and Lemma \ref{lem:tail_kappa}. 
Let us now show that for almost all $x\in \cX_+$,
\begin{equation}\label{eq:c_concave}
v^c(x) \defeq \sup_{y\in \R} (v(y)-c(x,y)) = v(f^*(x))- c(x,f^*(x)). 
\end{equation}
Let $x\in \cX_+$ be a point such that $y\mapsto g_y^{\kappa(y)}(x)$ is nonincreasing (almost all points satisfy this condition by nestedness). 
 We remark that $\partial_y c(x,y) = \frac{2(y-\eta(x))}{\Delta(x)}$. Hence,
\begin{align*}
c(x,y)-c(x,f^*(x)) &= \int_{f^*(x)}^y \frac{2(t-\eta(x))}{\Delta(x)}\dd t \\
&= -2\int_{f^*(x)}^y \p{\frac{\eta(x)-t}{\Delta(x)}-\kappa(t)}\dd t  -2\int_{f^*(x)}^y \kappa(t) \dd t \\
&= -2\int_{f^*(x)}^y \p{\frac{\eta(x)-t}{\Delta(x)}-\kappa(t)}\dd t  + v(y)-v(f^*(x)).
\end{align*}
Assume that $y\geq f^*(x)$. For $t\in (f^*(x),y]$, by definition of $f^*$ and by nestedness, $g_t^{\kappa(t)}(x)=0$. Thus, 
\[ \frac{\eta(x)-t}{\Delta(x)}-\kappa(t)< 0.\]
This implies that
\[ - 2\int_{f(x)}^y \p{\frac{\eta(x)-t}{\Delta(x)}-\kappa(t)}\dd t \geq 0.\]
We obtain that
\[ c(x,y)-c(x,f^*(x)) \geq v(y)-v(f^*(x)).\]
The same result holds when $y<f^*(x)$. Indeed, in that case, for all $t\in [y,f^*(x))$, 
\[ \frac{\eta(x)-t}{\Delta(x)}-\kappa(t)\geq 0. \]
Hence, 
\[ - 2\int_{f^*(x)}^y \p{\frac{\eta(x)-t}{\Delta(x)}-\kappa(t)}\dd t = 2\int_y^{f^*(x)} \p{\frac{\eta(x)-t}{\Delta(x)}-\kappa(t)}\dd t \geq 0.\]
This proves \eqref{eq:c_concave}. This relation implies that the $\bc$-transform of the function $v\in L^1(\nu^*)$ is a function $w:\Omega\to \R$ satisfying for $\mu_+$-almost all $x\in \cX_+$ (with $\x=(\eta(x),\Delta(x))$)
\[ w(\x) = v(f^*(x))- c(x,f^*(x)) = v(\f^*(\x))-\bc(\x,\f^*(\x)). \]

As $v\in L^1(\nu^*)$, Kantorovich duality  implies
\begin{align*}
    \OT_{\bc}(\mup,\nu^*) &\geq \int v(y)\dd\nu^*(y) -\int w(\x) \dd \mup(\x)= \int v(\f^*(\x))\dd\mup(\x)-\int w(\x) \dd \mup(\x) \\
    &= \int \bc(\x,\f^*(\x)) \dd \mup(\x),
\end{align*} 
see  \Cref{sec:OT}. 
This shows that $\f^*$ is the optimal transport map between $\mup$ and $\nu^*$. The same holds for $\mum$, where we use the potential $-v$ instead of $v$: precisely, we can show that we have for almost all $x\in \cX_-$
\begin{equation}\label{eq:c_concave_minus}
(-v)^c(x) \defeq \sup_{y\in \R} (-v(y)-c(x,y)) = -v(f^*(x))- c(x,f^*(x)). 
\end{equation}
This concludes the proof of Lemma \ref{prop:potential}.
\end{proof}

Lemma \ref{prop:potential} shows that $\f^*$ defines an optimal transport map from $\mup$ to $\nu^*$, and from $\mum$ to $\nu^*$. To conclude the proof of Theorem \ref{thm:awareness}, it remains to show that $\nu^*$ is solution to the barycenter problem described in Lemma \ref{lem:reducOT}.

\begin{lemma}\label{lem:nu_star_barycenter}
The distribution $\nu^*$ is solution to the barycenter problem described in Lemma \ref{lem:reducOT}.
\end{lemma}

\begin{proof}
Let $\phi:\x_1\in \Omega\mapsto \bc(\x_1,\f^*(\x_1))-v(\f^*(\x_1))$ and let $\psi:\x_2\in \Omega\mapsto \bc(\x_2,\f^*(\x_2)) + v(\f^*(\x_2))$. Using \eqref{eq:c_concave} and \eqref{eq:c_concave_minus}, we see that for all $y\in \R$, for  $\mup$-almost all $\x_1$ and $\mum$-almost all $\x_2$, it holds that
\begin{align*}
\phi(\x_1)+\psi(\x_2) &= \bc(\x_1,\f^*(\x_1))-v(\f^*(\x_1))+ \bc(\x_2,\f^*(\x_2)) + v(\f^*(\x_2))\\
&\leq \bc(\x_1,y)-v(y) + \bc(\x_2,y)+v(y) \\
&= \bc(\x_1,y)+\bc(\x_2,y).
\end{align*}
By taking the value $y$ that minimizes this last term, we obtain that
\[ \phi(\x_1)+\psi(\x_2)\leq \C(\x_1,\x_2),\]
where $\C$ is the cost function defined in \eqref{eq:cost_C}. In particular, $-\phi(\x_1)\geq \psi^\C(\x_1)$.
Furthermore, remark that
\[  -v(\f^*(\x_1))\leq \phi(\x_1)\leq c(\bx_1,0). \]
Thus, as $v\in L^1(\nu^*)$ and $\int \frac{h^2}{d}\dd \mup(h,d)<+\infty$, it holds that $\phi\in L^1(\mup)$. Likewise, $\psi\in L^1(\mum)$.  
By Kantorovich duality, it holds that
\begin{align*}
\OT_{\C}(\mup,\mum) &\geq \int \psi(\x_2)\dd \mum(\x_2)-\int \psi^\C(\x_1) \dd \mup(\x_1)\\
&\geq \int \psi(\x_2)\dd \mum(\x_2)+\int \phi(\x_1) \dd \mup(\x_1)\\
&= \int \bc(\x_1,\f^*(\x_1))\dd \mup(\x_1)+ \int \bc(\x_2,\f^*(\x_2))\dd \mum(\x_2) \\
&\quad - \int v(\f^*(\x_1))\dd \mup(\x_1) + \int v(\f^*(\x_2))\dd \mum(\x_2)  \\
&= \OT_{\bc}(\mup,\nu^*) + \OT_{\bc}(\mum,\nu^*) + \int v(y)\dd \nu^*(y)-\int v(y)\dd \nu^*(y) \\
&= \OT_{\bc}(\mup,\nu^*) + \OT_{\bc}(\mum,\nu^*) \geq \OT_\C(\mup,\mum).
\end{align*}
This proves that $\nu^*$ is the solution to the barycenter problem, and that $f^*$ is an optimal regression function.
\end{proof}

\section{Building examples and counterexamples}\label{sec:examples}
In the previous section, we  proved that under mild assumptions, the relationship $g^*_y(x) = \ones \{f^*(x) \geq y\}$ only holds under the nestedness assumption. In this section, we now explain how to build large classes of triplets $(X,Y,S)\in \cX\times \R\times \{1,2\}$ whose distributions $\P$ either satisfy or do not satisfy this criterion. The starting point of our approach consisted in associating to each distribution $\P$ a pair of distributions $(\mup,\mum)=(\mup(\P),\mum(\P))$ on $\Omega$, where we recall that $\mupm(\P)$ is the distribution of $(\eta(X),  \Delta(X) )$ when $X \sim \mu_{\pm}$. Then, both the optimal fair regression function and the nestedness criterion are best understood in terms of the pair  $(\mup(\P),\mum(\P))$.

However, given a pair of measure $(\mup,\mum)$ on $\Omega$, it is not a priori clear whether there exists a triplet $(X,Y,S)\sim \P$ with $(\mup,\mum)=(\mup(\P),\mum(\P))$. We give a definitive answer to this problem by providing a list of necessary and sufficient conditions for the existence of such a probability distribution $\P$. We then use this theoretical result to build probability distributions $\P$ for which the associated fair classification problem is either nested or not nested.

Let $\P$ be the distribution of a triplet $(X,Y,S)\in \cX\times \R\times \{1,2\}$, with $\E[Y^2]<+\infty$. Let $(\mup,\mum)=(\mup(\P),\mum(\P))$ be the associated pair of measures on $\Omega$.  Then, it always holds that
\[ \int_\Omega |d|^{-1}\dd \mup(h,d) = \int_{\cX} \frac{1}{\Delta(x)} \dd \mu_+(x) = \int_{\cX}\frac{\dd \mu}{\dd \mu_+}(x) \dd \mu_+(x) =\mu(\cX_+), \]
while $\int_\Omega| d|^{-1}\dd \mum(h,d)=\mu(\cX_-)$. In particular, 
\begin{equation}\label{eq:bound_on_inverse_d}
  0< \int_\Omega |d|^{-1}\dd \mup(h,d)+\int_\Omega |d|^{-1}\dd \mum(h,d) \leq 1. 
\end{equation}
Also, note that $\mu=p_1\mu_1+p_2\mu_2$, so that $\Delta(x) = \frac{\dd\mu_+}{\dd\mu}(x)\leq \frac{1}{p_1m}$ when $x \in \cX_+$, whereas $ \Delta(x)\geq- \frac{1}{p_2m}$ when $x\in \cX_-$ (recall that $m$ is the mass of the measure $(\mu_1-\mu_2)_+$). In particular, the supports of $\mup$ and $\mum$ are located in an horizontal strip of the form $\{(h,d):\ -M\leq d\leq M\}$ for some $M>0$. The next proposition states that these two conditions are actually sufficient for the existence of a probability measure $\P$ with $(\mup,\mum)=(\mup(\P),\mum(\P))$.

\begin{proposition}\label{prop:how_to_build_examples}
Assume that $\cX$ is an uncountable standard Borel space (e.g., $\cX=[0,1]$). Then, the set of pairs of measures $(\mup(\P),\mum(\P))$ that can be obtained from a distribution $\P$ of a triplet  $(X,Y,S)\in \cX\times \R\times \{1,2\}$ with $\E[Y^2]<\infty$ is exactly equal to the set of  pairs $(\mup,\mum)$ supported on bounded horizontal strips, satisfying \Cref{eq:bound_on_inverse_d}, with $\mup$ supported on $\{d>0\}$ and $\mum$ supported on $\{d<0\}$.
\end{proposition}
This proposition allows us to easily build examples where either nestedness or nonnestedness is satisfied: one does not need to build from scratch a joint distribution on $\cX\times \R\times \{1,2\}$, but can simply define a pair of measures $(\mup,\mum)$ on $\Omega$. As long as this pair satisfies the conditions given in Proposition \ref{prop:how_to_build_examples}, the existence of a probability distribution $\P$ such that $(\mup,\mum)=(\mup(\P),\mum(\P))$ is ensured. 
\medskip

\begin{proof}
We have already established that the  pairs of measures $(\mup(\P),\mum(\P))$ satisfy the conditions stated in Proposition \ref{prop:how_to_build_examples}. Reciprocally, consider a pair $(\mup,\mum)$ satisfying \Cref{eq:bound_on_inverse_d}, supported on bounded horizontal strips, with $\mup$ supported on $\{d>0\}$ and $\mum$ supported on $\{d<0\}$. 
Let $a_{\pm} = \int_\Omega |d|^{-1}\dd \mupm(h,d)$. 

Due to the Borel isomorphism theorem, $\cX$ is Borel isomorphic to $\R^2$, so we may assume without loss of generality that $\cX=\R^2$.  Let $\cX_+ = \{(h,d)\in \R^2:\ d>0\}$, $\cX_-=\{(h,d)\in \R^2:\ d<0\}$ and $\cX_{=} = \{(h,0):\ h\in \R\}$. Let $\mu_==\delta_0$. We let $\mu_+ = \mup$ and $\mu_-=\mum$.

Let
\begin{equation}
    \dd\mu(h,d) = \frac{1}{|d|}\dd\mu_+(h,d) + \frac{1}{|d|}\dd\mu_-(h,d) + (1-a_+-a_-)\dd\mu_=(h,d).
\end{equation}
Remark that $\mu$ is a probability measure:
\begin{align*}
    \int \dd \mu& = \int  \frac{1}{|d|}\dd \mu_+(h,d) + \int\frac{1}{|d|}\dd\mu_-(h,d) + (1-a_+-a_-)\int\dd\mu_= =1.
\end{align*} 
Consider $m$ small enough so that the inequality $m|d|/2\leq 1$  holds on the support of $\mu$ (this is possible because the $d$ coordinate is bounded in the support of $\mup$ and $\mum$). We define
\begin{equation}
    \dd \mu_1(h,d)  = (1+\frac{m}{2}d) \dd \mu(h,d) \quad \text{and} \quad  \dd \mu_2(h,d)  = (1-\frac{m}{2}d) \dd \mu(h,d)
\end{equation}
Note that $\mu = \frac 12\mu_1+\frac 12 \mu_2$. Also, $\mu_1$ and $\mu_2$ are probability measures, as
\[ \int d \dd \mu = \int \frac{d}{|d|} \dd \mu_+(h,d) + \int \frac{d}{|d|} \dd \mu_-(h,d) = 1-1=0.  \]
Let $\eta(h,d)=h$.  
We define the triplet $(X,Y,S)$ by letting $S$ be uniform on $\{1,2\}$. If $S=1$, we draw $X\sim \mu_1$ and let $Y=\eta(X)$. If $S=2$, we draw $X\sim \mu_2$ and let $Y=\eta(X)$. Let $\P$ be the distribution of $(X,Y,S)$. One can easily check that $(\mup(\P),\mum(\P))=(\mup,\mum)$, as desired.
\end{proof}

To build examples of nested and non-nested problems, we consider probability measures $\mup$ and $\mum$ supported on small horizontal segments:
\begin{equation}\label{eq:simple_mupm}
    \mupm = \frac{1}{K} \sum_{i=1}^K \nu_{\pm}^{(i)}
\end{equation}
where $\nu_\pm^{(i)}$ is the uniform measure on $[a^{(i)}_\pm,a^{(i)}_\pm+1]\times \{d_\pm^{(i)}\}$. 


\begin{figure}
    \centering
    \includegraphics[width=0.45\linewidth]{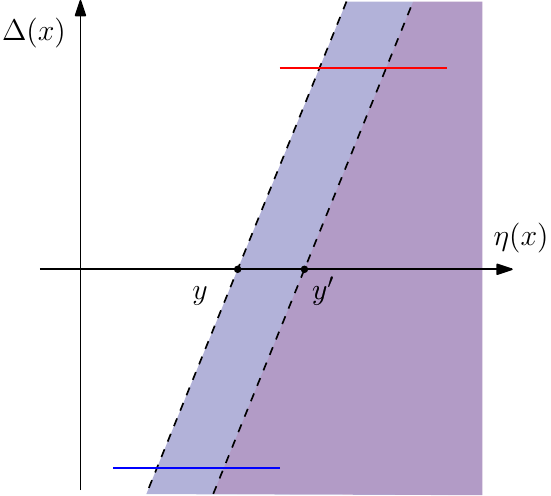}
    \hfill 
        \includegraphics[width=0.45\linewidth]{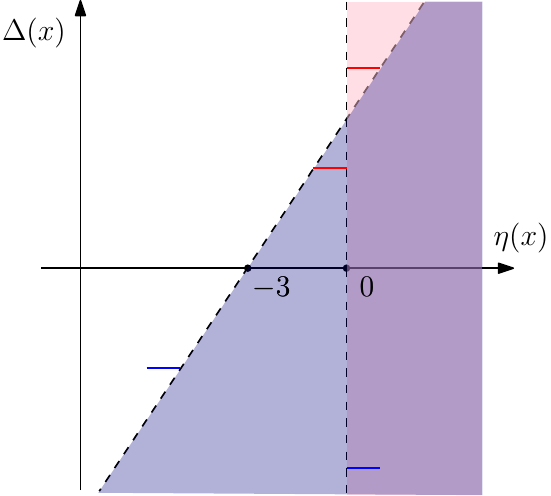}
    \caption{Left: example of a nested problem.  The distributions of $\mup$ and $\mum$ are depicted in red and blue, corresponding to the distributions given in Example \ref{ex:1}. The acceptance region for $g_y^{\kappa(y)}$ and $g_{y'}^{\kappa'(y)}$ are so that the masses of $\mup$ and $\mum$ to the right of the decision boundaries are equal. One can observe that these two regions are nested. Right: example of a non-nested problem. The distributions $\mup$ and $\mum$ are the ones described in Example \ref{ex:2}. The region in pink is rejected for $y=-3$ but accepted for $y=0$, contradicting the nestedness assumption.}
    \label{fig:example_nestedness}
\end{figure}

\begin{example}[A nested classification problem]\label{ex:1}
    Take $K=1$, $d_+^{(1)}=d_-^{(1)}=1$ and $a^{(1)}_+=0$, $a^{(1)}_-=-1$. Let $\P$ be the probability associated with the pair $(\mup,\mum)$ defined for this choice of parameters. Then, it holds that $1/2 \in I(y)$ for all $y\in \R$. By choosing $\kappa(y)=1/2$ for all $y\in \R$, we see that the classification problem associated with $\P$ is nested. See also Figure \ref{fig:example_nestedness}.
\end{example}

\begin{example}[A non-nested classification problem]\label{ex:2}
    Take $K=2$. Let $d_+^{(1)}=d_-^{(1)}=1$ and $a_+^{(1)}=a_-^{(1)}=0$. Let $d_+^{(2)}=d_-^{(2)}=1/2$, and $a_+^{(2)}=-1$, $a_-^{(2)}=-6$. Then, for $y=0$, $I(y)=\{0\}$, so the support of $\nu_+^{(2)}$ is to the left of the classification threshold for $y=0$. But for $y=-3$, we have $I(y)=\{4 \}$ and the support of $\nu_+^{(2)}$  is to the right of the classification threshold. Hence, the classification problem is non-nested. See also Figure \ref{fig:example_nestedness}.
\end{example}


\section{Conclusion and future work}
This work presents the first theoretical characterization of the optimal fair regression function as the solution to a barycenter problem with an optimal transport cost. Our results also demonstrate that, under the nestedness assumption, the optimal fair regression function can be represented by the family of classifiers $g_y^{\kappa(y)}$. Although both approaches—whether based on optimal transport or cost-sensitive classifiers—depend on the underlying distribution 
$\mathbb{P}$ which is generally unknown, they pave the way for developing new algorithms that estimate these unknown quantities from observed data. Designing such estimators, along with bounding their excess risk and potential unfairness, represents a critical step toward the development of fair algorithms.

While this work provides an initial characterization of the optimal fair regression function in the unawareness framework, it also has notable limitations. For instance, our results are currently limited to cases where the sensitive attribute is binary and apply only to univariate regression. Addressing these limitations and extending our findings to more general cases would be a valuable direction for future research.


\section*{Acknowledgements}
The authors gratefully acknowledge valuable and insightful discussions with Evgenii Chzhen and Nicolas Schreuder. S. G. gratefully acknowledges funding from the Fondation Mathématique Jacques Hadamard and from the ANR TopAI chair (ANR–19–CHIA–0001).
\bibliography{biblio}
\bibliographystyle{alpha}

\newpage

\appendix
\section{Additional proofs}
\subsection{Proof of Lemma \ref{lem:regularity}}\label{app:regularity}

Let $f:\R \to \R\cup\{+\infty\}$ be a lower semicontinuous convex function. The domain of such a function is an interval $\dom(f)$. Its right derivative $f'_+$ is defined and finite everywhere on $\dom(f)$, except on the right endpoint of the interval (should the right endpoint be included in $\dom(f)$) where it is equal to $+\infty$. 
Such a function is  upper semicontinuous, with the representation:
\begin{equation}\label{eq:left_derivative}
    \forall h\in \dom(f),\ f'_+(h) = \inf_{u>h}\frac{f(u)-f(h)}{u-h},
\end{equation}
where the infimum can be restricted to a countable dense collection of values $u$ if needed.
Likewise, the right derivative $f'_-$ can be defined on $\dom(f)$, and is lower semicontinuous. Note also that the oscillation function $\osc(f)=f'_+-f'_-\in [0,+\infty]$ can be defined on $\dom(f)$, and is upper semicontinuous. Indeed, only one of $f'_+$ and $f'_-$ can be infinite on $\dom(f)$ (and only at one of the endpoints of the domain), so that the difference is well defined. 

Recall that a function $\phi$ is $\C$-convex if
\begin{equation}\label{eq:def_phi}
\forall (h,d)\in \Omega,\ \phi(h,d) = \sup_{(h',d')\in \Omega}\p{ \phi^{\C}(h',d') - \frac{(h-h')^2}{|d|+|d'|}}.
\end{equation}
The function $\phi$ is lower semicontinuous as a supremum of continuous functions. Furthermore,  for any $d\neq 0$, the function
\[ \phi_{d}: h\mapsto \phi(h,d)+ \frac{h^2}{|d|} =\sup_{(h',d')\in \Omega}\p{ \phi^{\C}(h',d') + \frac{h^2}{|d|}- \frac{(h-h')^2}{|d|+|d'|} }\]
is convex as a supremum of lower semicontinuous convex functions. 
Let $G:\dom(\phi)\to [0,+\infty]$ be defined as $G(h,d)=\osc(\phi_d)(h) $ for $(h,d)\in \dom(\phi)$.  
Let $L>0$, and consider the  set  $\Sigma_{d,L}$ defined as  the set of points $h\in \dom(\phi_d)$ such that  $G(h,d)\geq L^{-1}$, $(\phi_d)'_+(h)\geq -L$ and $(\phi_d)'_-(h)\leq L$. As the left and right derivatives of $\phi_d$ are nondecreasing, the set $\Sigma_{d,L}$ is finite and its cardinality is bounded by a constant depending only on $L$. 
Let $\Sigma_L = \bigcup_{d\neq 0} \Sigma_{d,L}$ and $\Sigma=\bigcup_{L\in \N}\Sigma_L$. 
The set of points $\x\in \dom(\phi)$ such that $\partial_h\phi(\x)$ does not exist is equal to $\Sigma$. 
 Let us show that for any integer $L$, the set $\Sigma_L$ is included in a countable union of graphs of measurable functions. 
 

 To do so, we use the following general result, see \cite[Corollary 18.14]{Aliprantis1999}. A correspondence $\Phi$ from a measurable set $S$ to a topological space $X$ assigns to each $s\in S$ a subset $\Phi(s)$ of $X$. We say that the correspondence is (weakly) measurable if for each open subset $U\subset X$, the set $\Phi^\ell(U)=\{s\in S:\ \Phi(s)\cap U\neq \emptyset\}$ is measurable.
 
\begin{theorem}[Castaing's theorem]
    If $X$ is a Polish space and $\Phi$ is a measurable correspondence with non-empty closed values  between $S$ and $X$, then there exists a sequence $(f_n)_n$ of measurable functions $S\to X$ such that for every $s\in S$, $\Phi(s)=\overline{\{f_1(s),f_2(s),\cdots\}}$. 
\end{theorem}

Let $\Phi$ be the correspondence that assigns to each $d\neq 0$ the subset $\Sigma_{d,L}\cup\{0\}$ of $\R$. As each set $\Sigma_{d,L}$ is finite, this correspondence takes non-empty closed values. If we show that this correspondence is measurable, then Castaing theorem asserts the existence of a sequence of measurable functions $(f_n)_n$ such that for every $d\neq 0$, $\Sigma_{d,L}\cup\{0\} = \overline{\{f_1(d),f_2(d),\cdots\}}$. For each $d$, the set $\Sigma_{d,L}$ is finite, so that $\Sigma_{d,L}\cup\{0\} =  \{f_1(d),f_2(d),\cdots\}$, implying that $\Sigma_L$ is included in a countable union of graphs of measurable functions. It remains to show the measurability of $\Phi$.


\begin{lemma}\label{lem:meas}
 The function $G$ is measurable.
\end{lemma}
\begin{proof}
The representation \eqref{eq:left_derivative} implies that $(h,d)\mapsto (\phi_d)'_+(h)$ is given by a countable infimum of measurable functions, and is therefore measurable. Likewise, $(h,d)\mapsto (\phi_d)'_-(h)$ is measurable, so that $G$ is also measurable.
\end{proof}

Let $U\subset \R$ be an open set.  If $0\in U$, then $\Phi^\ell(U)=\R\backslash\{0\}$ is measurable. If $0\not \in U$, we have
\[ \Phi^\ell(U) =\{d\neq 0:\ \exists h\in [-L,L]\cap U,\ G(h,d)\geq L^{-1},\ (\phi_d)'_+(h)\geq -L,\  (\phi_d)'_-(h)\leq L \}. \]
This set is the projection on the $d$-axis of the measurable set 
\[
B=\{(h,d)\in \Omega:\ h\in [-L,L]\cap U,\ G(h,d)\geq L^{-1},\ (\phi_d)'_+(h)\geq -L,\  (\phi_d)'_-(h)\leq L \}
\]
Furthermore, for each $d$, the section $\{h\in \R:\ (h,d)\in B\}=\Sigma_{d,L}$ is compact. By \cite[Theorem 4.7.11]{srivastava2008course}, this implies that $\Phi^\ell(U)$ is measurable, concluding the proof of Lemma \ref{lem:regularity}.

\subsection{Proof of Proposition \ref{prp:classif}}\label{subsec:proof_classif}

Classical manipulations show that the risk $\cR_y(g)$ of a classifier $g$ can be expressed as
\begin{align*}
\mathcal{R}_y(g) &= y\mathbb{E}\left[(1-Y)g(X)\right] + (1-y) \mathbb{E}\left[Y(1-g(X))\right]\\
& = (1-y)\mathbb{E}[Y] + \mathbb{E}\left[g(X)\left(y-\eta(X)\right)\right].
\end{align*}
Using the definition of $\mu_+$, $\mu_-$ and $\Delta$ given in Section \ref{sec:barycenter}, we find that 
\begin{align*}
    &\mathbb{E}\left[g(X)\left(y-\eta(X)\right)\right]\\
    \qquad &= \int_{\mathcal{X}_+} g(x)(y-\eta(x))\frac{\dd\mu}{\dd\mu_+}(x) \dd\mu_+(x) + \int_{\mathcal{X}_-} g(x)(y-\eta(x))\frac{\dd\mu}{\dd\mu_-}(x) \dd\mu_-(x) \\
    \qquad &\qquad+ \int_{\mathcal{X}_=} g(x)(y-\eta(x)) \dd\mu(x)\\
    \qquad &= \int_{\mathcal{X}_+} g(x)\frac{y-\eta(x)}{|\Delta(x)|} \dd\mu_+(x) + \int_{\mathcal{X}_-} g(x)\frac{y-\eta(x)}{|\Delta(x)|} \dd\mu_-(x) + \int_{\mathcal{X}_=} g(x)(y-\eta(x)) \dd\mu(x).
\end{align*}
Moreover, Lemma \ref{lem:EvNico} implies that the demographic parity constraint is equivalent to the constraint $\mathbb{E}_{X\sim \mu_+}[g(X)] = \mathbb{E}_{X\sim \mu_-}[g(X)]$. Using the decomposition $g = \cF(g_+, g_-, g_=)$, we see that the fair classification problem can be rephrased as follows
\begin{equation}\tag{$C_y'$}
        \begin{cases}
             \text{minimize} & \mathbb{E}_{ \mu_+}\left[g_+(X)\frac{y-\eta(X)}{\Delta(X)}\right] -\mathbb{E}_{ \mu_-}\left[g_-(X)\frac{y-\eta(X)}{\Delta(X)}\right]  \\
             &\quad + \mathbb{E}_{ \mu}\left[\mathds{1}_{\cX_=}
             (X)g_+(X)(y-\eta(X))\right] \\
             \text{such that} & \mathbb{E}_{ \mu_+}[g_+(X)] = \mathbb{E}_{ \mu_-}[g_-(X)].
        \end{cases}
\end{equation}
The following lemma characterizes the solutions to the problem $(C_y')$.

\begin{lemma}\label{lem:aux_classif}
Under Assumption \ref{ass:continuity}, for any optimal classifier $g$, there exist $\kappa^+$, $\kappa^-$ such that $g= \cF(g^{\kappa^+}, g^{\kappa^-}, g_=)$, with 
\begin{align*}
    &g_=(x) = \mathds{1}\{\eta(x) > y\} \quad \text{ or }\quad g_=(x) = \mathds{1}\{\eta(x) \geq y\},\\
    \text{and}\quad   &g^{\kappa}(x) = \mathds{1}\left\{\eta(x) \geq y+ \kappa \Delta(x)\right\}.
\end{align*}
\end{lemma}
To conclude the proof of Proposition \ref{prp:classif}, it remains to prove that all optimal classifier are a.s. equal when $\Delta(X)\neq 0$, and that the optimal classifier can be chosen as $g^* = \cF(g^{\kappa^*}, g^{\kappa^*}, g_=)$ for some $\kappa^*$.

Denote by $F_+$ the c.d.f. of the random variable $Z_+ = \frac{\eta(X)-y}{\Delta(X)}$ when $X\sim \mu_+$ and by $F_-$ the c.d.f. of $Z_-= \frac{y-\eta(X)}{\Delta(X)}$ when $X\sim \mu_-$. Let $\cQ_+$ (resp. $\cQ_-$) be the associated quantile function. To verify the demographic parity constraint, the classifier $\cF(g^{\kappa^+}, g^{\kappa^-}, g_=)$ must be such that
$$F_+(\kappa^+) = F_-(-\kappa^-)$$
(recall that $\Delta(X)<0$ when $X\sim \mu_-$, so that $g^{\kappa^-}(X)=1$ if and only if $Z_- \geq -\kappa^-$).
Denoting $\beta = F_+(\kappa^+)= F_-(-\kappa^-)$ and using the definition of the quantile function, we see that the law of $g^{\kappa^{\pm}}(X)$, where $X \sim \mu_{\pm}$ is equal to the law of 
\[ \ones\{U\geq \beta\} = \ones\{ \cQ_{\pm}(U)\geq \pm\kappa^{\pm}\},  \]
where $U$ is a uniform random variable on $[0,1]$.
 
Then, minimizing the risk of the classifier is equivalent to maximizing
\begin{equation}\label{eq:function_to_be_maximized}
    \mathbb{E}\left[\left(\cQ_+(U) + \cQ_-(U)\right)\mathds{1}\left\{U\geq  \beta \right\}\right] 
\end{equation} 
Since $F_+$ and $F_-$ are continuous, $u \rightarrow \cQ_+(u)+\cQ_-(u)$ is strictly increasing and left-continuous. Straightforward computations show that the expression in \eqref{eq:function_to_be_maximized} has a unique maximum, which is attained for 
\begin{align*}
    \beta^* = \max \left\{\beta : \cQ_+(\beta)+\cQ_-(\beta) \leq 0 \right\}.
\end{align*}
Hence, it holds that $F_+(\kappa^+)=F_-(-\kappa^-)=\beta^*$. Let $g^*_1 = \cF(g^{\kappa^+_1}, g^{\kappa^-_1}, g_=)$ and $g^*_2 = \cF(g^{\kappa^+_2}, g^{\kappa^-_2}, g_=)$ be two optimal classifiers. For $\Delta(X)> 0$, they take different values only if $Z_+ \in [\kappa^+_1, \kappa^+_2]$. As $F_+(\kappa^+_1)=F_+(\kappa^+_2)=\beta^*$, this happens with zero probability. Likewise, the two classifiers are a.s. equal when $\Delta(X)<0$.

It remains to show that we can pick $\kappa^+=\kappa^-$. If $\cQ_++\cQ_-$ is continuous at $\beta^*$, the proof is complete: in this case, $\cQ_+(\beta^*)+\cQ_-(\beta^*)=0$, and the choice $\kappa^* = \cQ_+(\beta^*) = - \cQ_-(\beta^*)$ satisfies $F_+(\kappa^*)=F_-(-\kappa^*)=\beta^*$.

Otherwise, $\cQ_+(\beta^*)+\cQ_-(\beta^*)<0$. Defining 
\[ q_+ = \liminf_{\beta \rightarrow \beta^*_+} \cQ_+(\beta)\ \text{ and }\ q_- = \liminf_{\beta \rightarrow \beta^*_+} \cQ_-(\beta),\]
we have $q_+ +  q_->0$. Because $q_+>-q_-$ and $\cQ_+(\beta^*)<-\cQ_-(\beta^*)$, there exists $\kappa^*\in [\cQ_+(\beta^*), q_+] \cap [-q_-, - \cQ_-(\beta^*)]$. By construction, $F_+(\kappa^*)=F_-(-\kappa^*)=\beta^*$, concluding the proof.

\subsection{Proof of Lemma \ref{lem:aux_classif}}

Let $g^*$ be a solution to the problem $(C_y')$, and let $g_+^*$, $g_-^*$ and $g_=^*$ be the restrictions of $g^*$ to $\cX_+$, $\cX_=$ and $\cX_=$, so that $g^* = \cF(g_+^*, g_-^*, g_=^*)$. Straightforward computations show that  we necessarily have $g_=^*(X) = \mathds{1}\{\eta(X) \geq y\}\mathds{1}_{\cX_=}(X)$ or $g_=^*(X) = \mathds{1}\{\eta(X) > y\}\mathds{1}_{\cX_=}(X)$.

Now, let us assume (without loss of generality) that $g_+^*$ is not of the form $g^{\kappa^+}$. More precisely, assume that for $\kappa^+$ such that $\mathbb{E}_{\mu_+}[g^{\kappa^+}(X)] = \mathbb{E}_{\mu_+}[g^*_+(X)]$, we have $g^{\kappa^+}(X)\neq g^*_+(X)$ with positive $\mu_+$-probability. Then, the classifier $\cF(g^{\kappa^+}, g^*_-, g^*_=)$ verifies the demographic parity constraint. Moreover, we have
\begin{align*}
    &\mathbb{E}_{\mu_+}\left[g^{\kappa^+}(X)\frac{y-\eta(X)}{\Delta(X)} \right] - \mathbb{E}_{\mu_+}\left[g_+^*(X)\frac{y-\eta(X)}{\Delta(X)} \right] \\
    &= \mathbb{E}_{\mu_+}\left[\frac{y-\eta(X)}{\Delta(X)} g^{\kappa^+}(X)(1-g^*_+(X)) - \frac{y-\eta(X)}{\Delta(X)} (1-g^{\kappa^+}(X))g^*_+(X)\right]\\
    &= \mathbb{E}_{\mu_+}\left[\left(\kappa^+- \frac{\eta(X)-y}{\Delta(X)}\right) g^{\kappa^+}(X)(1-g^*_+(X))\right] - \kappa^+\mathbb{E}_{\mu_+}\left[g^{\kappa^+}(X)(1-g^*_+(X))\right] \\
    &\quad - \mathbb{E}_{\mu_+}\left[\left(\kappa^+- \frac{\eta(X)-y}{\Delta(X)}\right) (1-g^{\kappa^+}(X))g^*_+(X)\right] + \kappa^+\mathbb{E}_{\mu_+}\left[(1-g^{\kappa^+}(X))g^*_+(X)\right].
\end{align*}
Since $\mathbb{E}_{\mu_+}[g^{\kappa^+}(X)] = \mathbb{E}_{\mu_+}[g^*_+(X)]$, we obtain that \[\mathbb{E}_{\mu_+}\left[g^{\kappa^+}(X)(1-g^*_+(X))\right] = \mathbb{E}_{\mu_+}\left[(1-g^{\kappa^+}(X))g^*_+(X)\right].\] Then, the definition of $g^{\kappa^+}$ implies 

\begin{align*}
    &\mathbb{E}_{\mu_+}\left[g^{\kappa^+}(X)\frac{y-\eta(X)}{\Delta(X)} \right] - \mathbb{E}_{\mu_+}\left[g_+^*(X)\frac{y-\eta(X)}{\Delta(X)} \right] \\
     &=- \mathbb{E}_{\mu_+}\left[\left(\kappa^+- \frac{\eta(X)-y}{\Delta(X)}\right)_-(1-g^*_+(X))\right] - \mathbb{E}_{\mu_+}\left[\left(\kappa^+- \frac{\eta(X)-y}{\Delta(X)}\right)_+g^*_+(X)\right] < 0.
\end{align*}
This implies that $\cR_y(\cF(g^{\kappa^+}, g_-^*, g_=^*)) <  \cR_y(\cF(g_+^*, g_-^*, g_=^*))$, which is absurd. Using a similar argument for $g_-^*$, we arrive to the conclusion that  any optimal classifier is of the form $\cF(g^{\kappa^+}, g^{\kappa^-}, g_=^*)$ with $g_=^*(X) = \mathds{1}\{\eta(X) \geq y\}\mathds{1}_{\cX_=}(X)$ or $g_=^*(X) = \mathds{1}\{\eta(X) > y\}\mathds{1}_{\cX_=}(X)$.

\subsection{Proof of Lemma \ref{lem:tail_kappa}}
\label{proof:kappa}
Let us first show that $\kappa$ is locally bounded. Let $[y_0,y_1]$ be a bounded interval. Then, if $\kappa(y)\geq M$ for some $y\in [y_0,y_1]$,
\begin{align*}
   F(y)&= \mu_+\p{\eta(X)\leq y+\kappa(y)\Delta(X)}\geq \mu_+\p{\eta(X)\leq y_0+M\Delta(X)} 
\end{align*}
and
\begin{align*}
   F(y)&= \mu_-\p{\eta(X)\leq y+\kappa(y)\Delta(X)}\leq \mu_-\p{\eta(X)\leq y_1+M\Delta(X)}. 
\end{align*}
However, for $M$ large enough,
\[   \mu_-\p{\eta(X)\leq y_1+M\Delta(X)} < \mu_+\p{\eta(X)\leq y_0+M\Delta(X)},\]
and therefore it holds that $\kappa(y)\leq M$ for all $y\in [y_0,y_1]$. Likewise, we show that for $M$ large enough, $\kappa(y)\geq -M$ for all $y\in [y_0,y_1]$. Hence, $\kappa$ is locally bounded.

We refine this argument to obtain a control on $\kappa$ for large values of $y$. Let $L>1$ and let $y>1$ be such that $\kappa(y)\geq L y$. Then, because of the definition of $\kappa(y)$ and as $\Delta(X)<0$ for $X\sim \mu_-$,
    \begin{align*}
      \mu_+(\eta(X)\geq y)&\geq \mu_+(\eta(X)\geq y+ \kappa(y)\Delta(X)) \\
      &= \mu_-(\eta(X)\geq y+ \kappa(y)\Delta(X)) \geq   \mu_-(\eta(X)\geq y+ Ly\Delta(X)).
    \end{align*}
    The complementary of the region $\{(h,d)\in \Omega:\ d<0,\ h\geq y+Lyd\}$ in the lower half-plane $\{(h,d)\in \Omega:\ d<0\}$ is contained in the set $A_L$ given by the union of the horizontal strip $\{d\geq -1/L\}$ with the region $\{(h,d)\in \Omega:\ d<0,\ h\leq 1+Ld\}$. For $L$ large enough, $\mum(A_L)<1/2$. Then it holds that $\mu_+(\eta(X)\geq y)\geq 1/2$. For $y$ large enough, this is not possible. Thus, we have shown that there exist $L>0$ and $C>0$ such that for $y>C$, we have $\kappa(y)\leq Ly$. Likewise, we show that there exist constants $L,C>0$ such that for $|y|>C$, $|\kappa(y)|\leq L|y|$. As $\kappa$ is also locally bounded, the conclusion follows.

\end{document}